\documentclass[twoside,11pt]{article}
\usepackage{hyperref}
\usepackage{jair}
\usepackage{theapa}
\usepackage{rawfonts}

\newcommand{\note}[1]{}
\renewcommand{\note}[1]{~\\\frame{\begin{minipage}[c]{1\textwidth}\vspace{2pt}\center{#1}\vspace{2pt}\end{minipage}}\vspace{3pt}\\}

\newcommand{\hide}[1]{}

\usepackage{amsmath,amsthm,amssymb}

\usepackage{color}
\definecolor{Blue}{rgb}{0.9,0.3,0.3}

\newcommand{\squishlist}{
   \begin{list}{$\bullet$}
    { \setlength{\itemsep}{0pt}      \setlength{\parsep}{3pt}
      \setlength{\topsep}{3pt}       \setlength{\partopsep}{0pt}
      \setlength{\leftmargin}{1.5em} \setlength{\labelwidth}{1em}
      \setlength{\labelsep}{0.5em} } }

\newcommand{\squishlisttwo}{
   \begin{list}{$\bullet$}
    { \setlength{\itemsep}{0pt}    \setlength{\parsep}{0pt}
      \setlength{\topsep}{0pt}     \setlength{\partopsep}{0pt}
      \setlength{\leftmargin}{2em} \setlength{\labelwidth}{1.5em}
      \setlength{\labelsep}{0.5em} } }

\newcommand{\squishend}{
    \end{list}  }

\newcommand{\denselist}{\itemsep 0pt\topsep-6pt\partopsep-6pt}

\newcommand{\myvec}[1]{\mathbf{#1}}
\newcommand{\myvecsym}[1]{\boldsymbol{#1}}

\newcommand{\vDelta}{\myvecsym{\Delta}}

\newcommand{\vPhi}{\myvecsym{\Phi}}

\newcommand{\vPi}{\myvecsym{\Pi}}

\newcommand{\vxi}{\myvecsym{\xi}}

\newcommand{\va}{\myvec{a}}

\newcommand{\vc}{\myvec{c}}

\newcommand{\vf}{\myvec{f}}

\newcommand{\vk}{\myvec{k}}

\newcommand{\vm}{\myvec{m}}

\newcommand{\vv}{\myvec{v}}

\newcommand{\vx}{\myvec{x}}

\newcommand{\vy}{\myvec{y}}

\newcommand{\vz}{\myvec{z}}

\newcommand{\vA}{\myvec{A}}
\newcommand{\vB}{\myvec{B}}

\newcommand{\vD}{\myvec{D}}

\newcommand{\vG}{\myvec{G}}

\newcommand{\vI}{\myvec{I}}

\newcommand{\vK}{\myvec{K}}

\newcommand{\vM}{\myvec{M}}

\newcommand{\vO}{\myvec{O}}
\newcommand{\vP}{\myvec{P}}

\newcommand{\vR}{\myvec{R}}
\newcommand{\vS}{\myvec{S}}

\newcommand{\vU}{\myvec{U}}
\newcommand{\vV}{\myvec{V}}

\newcommand{\calD}{{\cal D}}

\newcommand{\calH}{{\cal H}}

\newcommand{\calS}{{\cal S}}

\newcommand{\calT}{{\cal T}}

\newcommand{\calX}{{\cal X}}
\newcommand{\calY}{{\cal Y}}

\newcommand{\data}{\calD}

\newcommand{\be}{\begin{equation}}
\newcommand{\ee}{\end{equation}}
\newcommand{\bea}{\begin{eqnarray}}
\newcommand{\eea}{\end{eqnarray}}
\newcommand{\beaa}{\begin{eqnarray*}}
\newcommand{\eeaa}{\end{eqnarray*}}

\DeclareMathAlphabet{\mathpzc}{OT1}{pzc}{m}{n}

\DeclareMathOperator*{\argmax}{arg\,max}
\newtheorem{mydefinition}{Definition}
\newtheorem{proposition}[mydefinition]{Proposition}
\newtheorem{theorem}[mydefinition]{Theorem}
\newtheorem{remark}[mydefinition]{Remark}
\newtheorem{lemma}[mydefinition]{Lemma}

\newcommand{\xbest}{\mathbf{\vx}^{+}}

\usepackage{graphicx} 
\usepackage{subfigure} 

\usepackage{algorithm}
\usepackage{algorithmic}

\ShortHeadings{Bayesian Optimization in a Billion Dimensions}
{Wang, Hutter, Zoghi, Matheson, \& de Freitas}
\firstpageno{1}

\title{Bayesian Optimization in a Billion Dimensions \\ via Random Embeddings}

\author{\name Ziyu Wang \email ziyu.wang@cs.ox.ac.uk \\
           \addr Department of Computer Science, 
						University of Oxford
       \AND
       \name Frank Hutter \email fh@cs.uni-freiburg.de \\
           \addr Department of Computer Science,
						University of Freiburg
			\AND \name Masrour Zoghi \email m.zoghi@uva.nl \\
       \addr 
						Department of Computer Science,
						University of Amsterdam
       \AND
       \name David Matheson \email davidm@cs.ubc.ca \\
           \addr Department of Computer Science,
						University of British Columbia
       \AND
       \name Nando de Freitas \email nando@cs.ox.ac.uk \\
           \addr Department of Computer Science,
            University of Oxford\\
            \addr Canadian Institute for Advanced Research
			}

\begin{document}
\maketitle


\begin{abstract}
Bayesian optimization techniques have been successfully applied to robotics, planning, sensor placement, recommendation, advertising, intelligent user interfaces and automatic algorithm configuration. Despite these successes, the approach is restricted to problems of moderate dimension, and several 
workshops on Bayesian optimization have identified its scaling to high-dimensions as one of the holy grails of the field. 
In this paper, we introduce a novel random embedding idea to attack this problem.
The resulting Random EMbedding Bayesian Optimization (REMBO) algorithm is very simple, has important invariance properties, 
and applies to domains with both categorical and continuous variables. 
We present a thorough theoretical analysis of REMBO. 
Empirical results confirm that REMBO can effectively solve problems with billions of dimensions, provided the intrinsic dimensionality is low. 
They also show that REMBO achieves state-of-the-art performance in optimizing the 47 discrete parameters of a popular mixed integer linear programming solver. 
\end{abstract}

\section{Introduction}
\label{sec:introduction}
Let $f: {\cal X} \to \mathbb{R}$ be a function on a compact subset ${\cal X} \subseteq \mathbb{R}^D$. We address the following global optimization problem
\[ \vx^{\star} = \argmax_{\vx \in {\cal X}} f(\vx). \]
We are particularly interested in objective functions $f$ that may satisfy one or more of the following criteria: they do not have a closed-form expression, are expensive to evaluate, do not have easily available derivatives, or are non-convex. We treat $f$ as a \emph{blackbox} function that only allows us to query its function value at arbitrary $x\in \cal{X}$.
To address objectives of this challenging nature, we adopt the Bayesian optimization framework. 

In a nutshell, in order to optimize a blackbox function $f$, Bayesian optimization uses a prior distribution that captures our beliefs about the behavior of $f$,
and updates this prior with sequentially acquired data.
Specifically, it iterates the following phases:
(1)~use the prior to decide at which input $x\in \cal X$ to query $f$ next; 
(2)~evaluate $f(x)$; and (3)~update the prior based on the new data $\langle{}x, f(x)\rangle$.
Step 1 uses a so-called \emph{acquisition function} that quantifies the expected value of learning the value of $f(x)$ for each $x \in \cal X$. This procedure is illustrated in Figure~\ref{Fig:BOpic}.
\begin{figure}[t]
\begin{center}
\includegraphics[width=13cm]{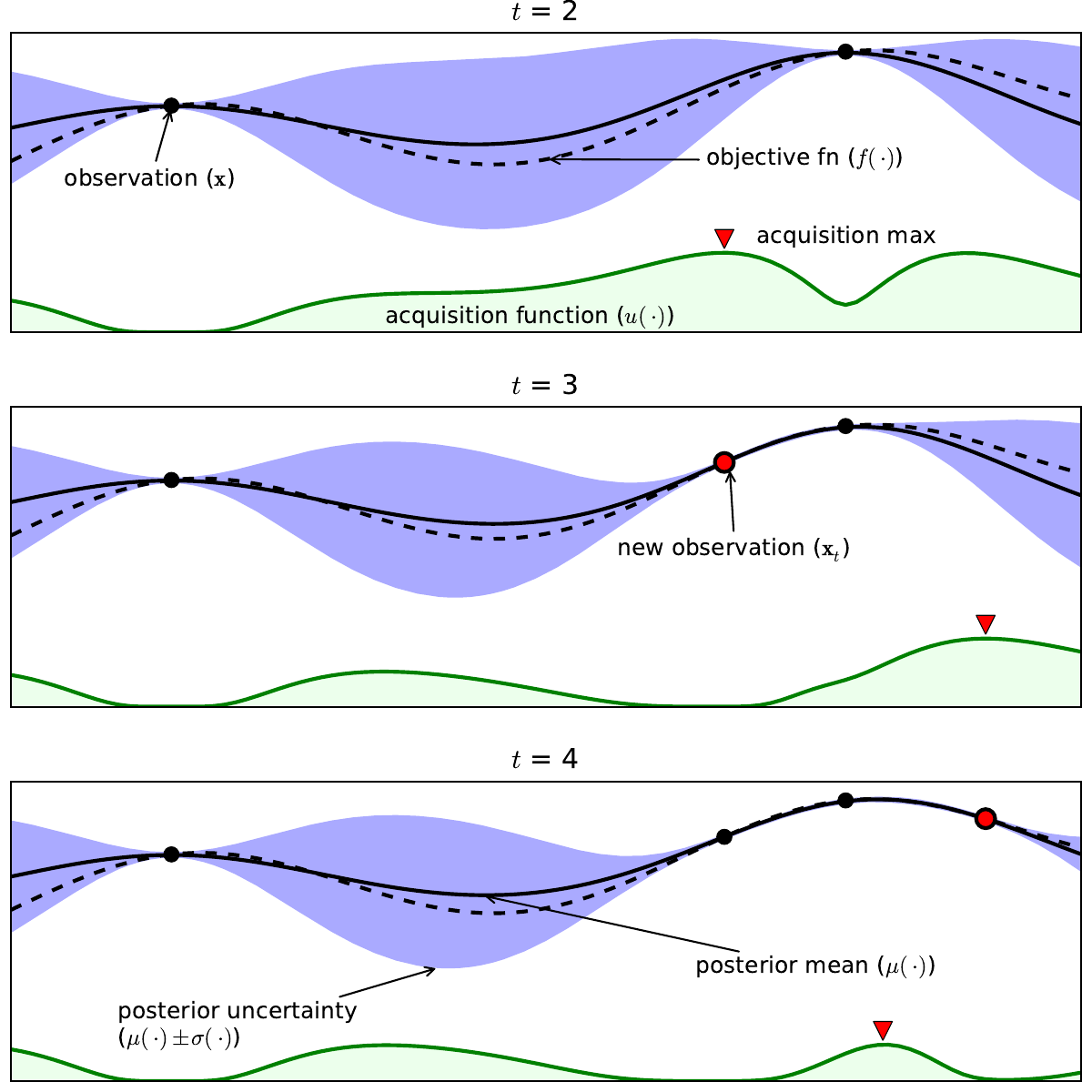}
\end{center}
\caption{Three consecutive iterations of Bayesian optimization for a toy one-dimensional problem. The unknown objective function is approximated with at Gaussian process (GP) at each iteration. The figure shows the mean and confidence intervals for this process.
It also shows the acquisition function in the lower green shaded plots. The acquisition is high where the GP predicts a high objective (exploitation) and where the prediction uncertainty is high (exploration). 
Note that the area on the far left remains under-sampled, as (despite having high uncertainty) it is correctly predicted to be unlikely to improve over the highest observation.}
\label{Fig:BOpic}
\end{figure}

%
%
The role of the acquisition function is to trade off exploration and exploitation; popular choices include Thompson sampling \cite{Thompson:1933,Hoffman:2014}, probability of improvement \cite{Jones:2001}, expected improvement \cite{Mockus:1994}, upper-confidence-bounds \cite{Srinivas:2010}, and online portfolios of these \cite{Hoffman:2011}.  
These are typically optimized by choosing 
points where the predictive mean is high (exploitation) and where the variance is large (exploration). Since they typically have an analytical expression that is easy to evaluate, they are much easier to optimize than the original objective function, using off-the-shelf numerical optimization algorithms.\footnote{This optimization step can in fact be circumvented when using treed multi-scale optimistic optimization as recently demonstrated by \citeA{Wang:2014}. There also exist several more involved Bayesian non-linear experimental design approaches for constructing the acquisition function, where the utility to be optimized involves an entropy of an aspect of the posterior. This includes the work of \citeA{Hennig:2012} for finding maxima of functions, the works of \citeA{Kueck:2006} and \citeA{Kueck:2009} for learning functions, and the work of \citeA{Hoffman:2009} for estimating Markov decision processes. These works rely on expensive approximate inference methods for computing intractable integrals.} 

The term \emph{Bayesian optimization} was coined several decades ago by Jonas Mo{\v c}kus \citeyear{Mockus:1982}. 
A popular version of the method is known as \emph{efficient global optimization} in the experimental design literature since the 
1990s \cite{Jones:1998}. 
Often, the approximation of the objective function is obtained using Gaussian process~(GP) priors. For this reason, the technique is also referred to as \emph{GP bandits} \cite{Srinivas:2010}. However, many other approximations of the objective have been proposed, including Parzen estimators \cite{Bergstra:2011}, Bayesian parametric models \cite{Wang:2011}, treed GPs \cite{Gramacy:2004} and random forests \cite{Brochu:2009,Hutter:2009,Hutter:2011}. 
These may be more suitable than GPs when the number of iterations grows without bound, or when the objective function is believed to have discontinuities. We also note that often assumptions on the smoothness of the objective function are encoded without use of the Bayesian paradigm, while leading to similar algorithms and theoretical guarantees \cite<see, for example,>[and the references therein]{Bubeck:2011}.
There is a rich literature on Bayesian optimization, and for further details we refer readers to more tutorial treatments \cite{Brochu:2009,Jones:1998,Jones:2001,Lizotte:2011,Mockus:1994,Osborne:2009} and recent theoretical results~\cite{Srinivas:2010,Bull:2011,deFreitas:2012}.

Bayesian optimization has been demonstrated to outperform other state-of-the-art blackbox optimization techniques when function evaluations are expensive and the number of allowed function evaluations is therefore low~\cite{HutHooLey13:BBOB}. In recent years, it has found increasing use in the machine learning community~\cite{Rasmussen:2003,Brochu:2007,martinez-cantin:2007,Lizotte:2008_IJCAI,Frazier:2009,Azimi:2010,Hamze:2011,Azimi:2011,Hutter:2011,Bergstra:2011,Gramacy:2011,Denil:2012,Mahendran:2012,Azimi:2012,Hennig:2012,Marchant:2012,Snoek:2012,Swersky:2013,Thornton:2013}. Despite many success stories, the approach is restricted to problems of moderate dimension, typically up to about 10. Of course, for a great many problems this is all that is needed. However, to advance the state of the art, we need to scale the methodology to high-dimensional parameter spaces. This is the goal of this paper.

It is difficult to scale Bayesian optimization to high dimensions. To ensure that a global optimum is found, we require good coverage of $\mathcal{X}$, but as the dimensionality increases, the number of evaluations needed to cover $\mathcal{X}$ increases exponentially. As a result, there has been little progress on this challenging problem, with a few exceptions.
\citeA{Bergstra:2011} introduced a non-standard Bayesian optimization method based on a tree of one-dimensional density estimators and applied it successfully to optimize the 238 parameters of a complex vision architecture \cite{Bergstra:model_search}.
\citeA{Hutter:2011} used random forests models in Bayesian optimization to achieve state-of-the-art performance in optimizing up to 76 mixed discrete/continuous parameters of algorithms for solving hard combinatorial problems, and to successfully carry out combined model selection and hyperparameter optimization for the 768 parameters of the Auto-WEKA framework~\cite{Thornton:2013}. 
\citeA{Eggensperger:2013} showed that these two methods indeed yielded the best performance for high-dimensional hyperparameter optimization (e.g., in deep belief networks). However, both are based on weak uncertainty estimates that can fail even for the optimization of very simple functions and lack theoretical guarantees. 

In the \emph{linear} bandits case, \citeA{Carpentier:2012} recently proposed a compressed sensing strategy to attack problems with a high degree of sparsity. 
Also recently, \citeA{Chen:2012} made significant progress by introducing a two stage strategy for optimization and variable selection of high-dimensional GPs. In the first stage, sequential likelihood ratio tests, with a couple of tuning parameters, are used to select the relevant dimensions. 
This, however, requires the relevant dimensions to be axis-aligned with an ARD kernel. Chen and colleagues provide empirical results only for synthetic examples (of up to 400 dimensions), but they provide key theoretical guarantees.

Many researchers have noted that for certain classes of problems most dimensions do not change the objective function significantly; examples include hyper-parameter optimization for neural networks and deep belief networks~\cite{Bergstra:2012}, as well as other machine learning algorithms and various state-of-the-art algorithms for solving $\mathcal{NP}$-hard problems~\cite{Hutter:2014}.
That is to say these problems have \emph{``low effective dimensionality''}. To take advantage of this property, \citeA{Bergstra:2012} proposed to simply use random search for optimization -- the rationale being that points sampled uniformly at random in each dimension can densely cover each low-dimensional subspace. As such, random search can exploit low effective dimensionality \emph{without knowing which dimensions are important}. In this paper, we exploit the same property in a new Bayesian optimization variant based on random embeddings. 

Figure \ref{fig:simple_embedding} illustrates the idea behind random embeddings in a nutshell.
Assume we know that a given $D=2$ dimensional black-box function $f(x_1, x_2)$ only has $d=1$ important dimensions, but we do not know which of the two dimensions is the important one.
We can then perform optimization in the embedded 1-dimensional subspace defined by $x_1=x_2$ since this is guaranteed to include the optimum. 

\begin{figure}[tb]
  \includegraphics[scale=0.32]{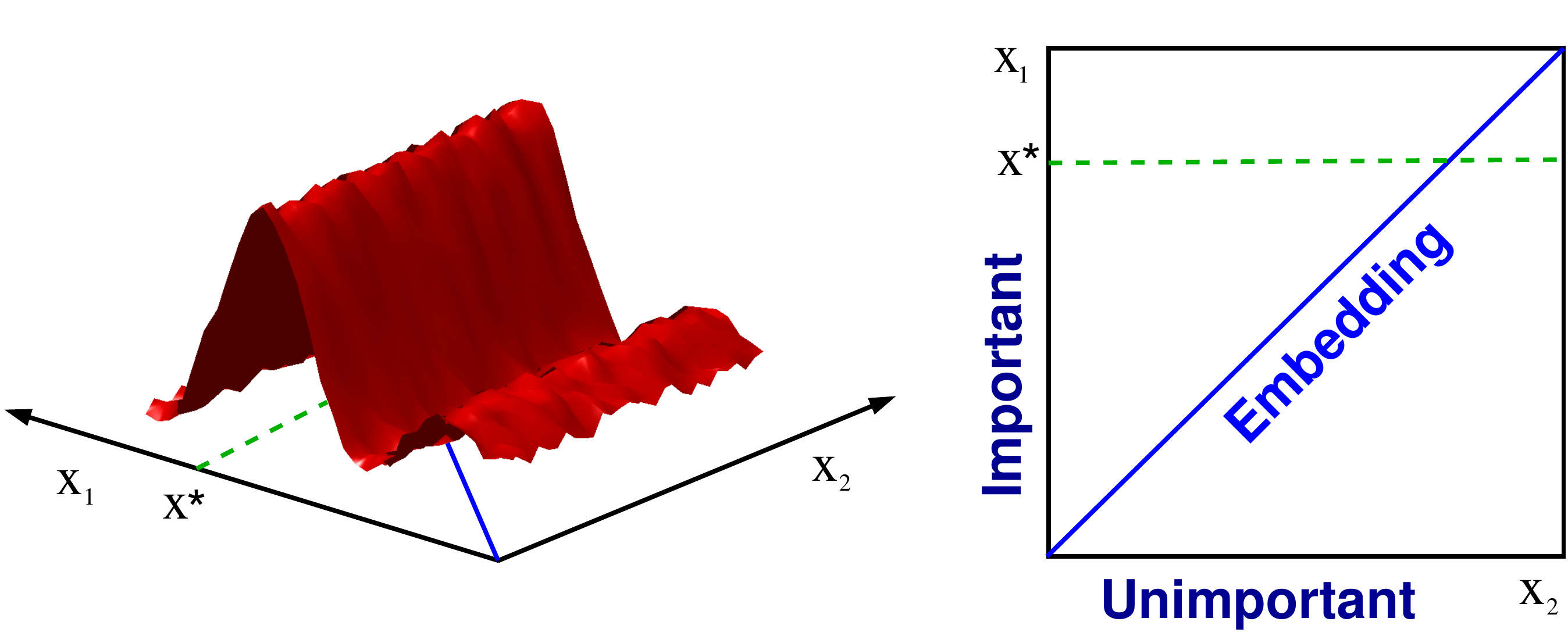}
  \centering
  \caption{This function in D=2 dimesions only has d=1 \emph{effective dimension}: the vertical axis indicated with the word important on the right hand side figure. Hence, the $1$-dimensional embedding includes the $2$-dimensional function's optimizer. It is more efficient to search for the optimum along the 1-dimensional random embedding than in the original 2-dimensional space.}
  \label{fig:simple_embedding}
\end{figure}

As we first demonstrated in a recent IJCAI conference paper~\cite{WangEtAl13}, random embeddings enable us to scale Bayesian optimization to arbitrary $D$ provided the objective function has low intrinsic dimensionality.
Importantly, the algorithm associated with this idea, which we called REMBO, is not restricted to cases with axis-aligned intrinsic dimensions but applies to any $d$-dimensional linear subspace. \citeA{Djolonga:2013} recently proposed an adaptive, but more expensive, variant of REMBO with theoretical guarantees.

In this journal version of our work, we expand the presentation to provide more details throughout.
In particular, we expand our description of the strategy for selecting the boundaries of the low-dimensional space and for setting the kernel length scale parameter; we show by means of an additional application (automatic configuration of random forest body-part classifiers) that the performance of our technique does not collapse when the problem does not have an obvious low effective dimensionality.
%
Our experiments (Section \ref{sec:experiments}) also show that REMBO can solve problems of previously untenable high extrinsic dimensions, and that REMBO can achieve state-of-the-art performance for optimizing the 47 discrete parameters of a popular mixed integer linear programming solver.


\section{Bayesian Optimization}\label{sec:bo}

As mentioned in the introduction, Bayesian optimization has two ingredients that need to be specified: The prior and the acquisition function. In this work, we adopt GP priors. We review GPs very briefly and refer the interested reader to the book by \citeA{Rasmussen:2006}.
A GP is a distribution over functions specified by its mean function $m(\cdot)$ and covariance $k(\cdot,\cdot)$. More specifically, given a set of points $\vx_{1:t}$, with $\vx_i \in \mathbb{R}^D$, we have 
$$
\vf(\vx_{1:t}) \sim \mathcal{N}(\vm(\vx_{1:t}), \vK(\vx_{1:t}, \vx_{1:t})),
$$
where $\vK(\vx_{1:t}, \vx_{1:t})_{i,j} = k(\vx_i, \vx_j)$
serves as the covariance matrix. A common choice of $k$ is the squared exponential function (see Definition~\ref{def:sekernel} on page~\pageref{def:sekernel}),
but many other choices are possible depending on our degree of belief about the smoothness of the objective function.

An advantage of using GPs lies in their analytical tractability. In particular, given observations $\vx_{1:t}$ with corresponding values $\vf_{1:t}$, where $f_i = f(\vx_i)$,  and a new point $\vx^{*}$, the joint distribution is given by:
$$
\begin{bmatrix}\vf_{1:t} \\
f^* \end{bmatrix} \sim \mathcal{N}\left( 
\begin{bmatrix}\vm(\vx_{1:t}) \\
m^* \end{bmatrix},  
\begin{bmatrix}
\vK(\vx_{1:t}, \vx_{1:t}) & \vk(\vx_{1:t}, \vx^*)\\
\vk(\vx^*, \vx_{1:t}) & k(\vx^*, \vx^*)\end{bmatrix}\right).
$$
For simplicity, we assume that $\vm(\vx_{1:t}) = \mathbf{0}$ and $m^* = 0$. Using the Sherman-Morrison-Woodbury formula, 
one can easily arrive at the posterior predictive distribution:
$$
f^* | \data_t, \vx^* \sim \mathcal{N}(\mu(\vx^*|\data_t), \sigma(\vx^*|\data_t)),
$$
with data $\data_t = \{\vx_{1:t}, \vf_{1:t} \}$, and mean and variance
\begin{eqnarray}
\nonumber{}\mu(\vx^*|\data_t) & = & \vk(\vx^*, \vx_{1:t}) \vK(\vx_{1:t}, \vx_{1:t})^{-1} \vf_{1:t}\\
\nonumber{}\sigma(\vx^*|\data_t) & = & k(\vx^*, \vx^*)-\vk(\vx^*, \vx_{1:t}) \vK(\vx_{1:t},\vx_{1:t})^{-1} \vk(\vx_{1:t}, \vx^*).\end{eqnarray}
That is, we can compute the posterior predictive mean $\mu(\cdot)$ and variance $\sigma(\cdot)$ exactly for any point $\vx^*$.

At each iteration of Bayesian optimization, one has to re-compute the predictive mean and variance. These two quantities are used to construct the second ingredient of Bayesian optimization: The acquisition function. In this work, we report results for the expected improvement acquisition function~\cite{Mockus:1982,Vazquez:2011,Bull:2011}:
\[\nonumber{}u(\vx|\mathcal{D}_t)=\mathbb{E}(\max\{0,f_{t+1}(\vx) - f(\xbest)\} |\data_t).\]
In this definition, $\xbest = \argmax_{\vx \in \{\vx_{1:t} \} }f(\vx)$
is the element with the best objective value in the first $t$ steps of the optimization process. The next query is:
\[\nonumber{}\vx_{t+1} = \argmax_{\vx \in {\cal X}} u(\vx|\mathcal{D}_t).\]
Note that this utility favors the selection of points with high variance (points in regions not well explored) and points with high mean value (points worth exploiting). We also experimented with the UCB acquisition function \cite{Srinivas:2010,deFreitas:2012} and found it to yield similar results. The optimization of the closed-form acquisition function can be carried out by off-the-shelf numerical optimization procedures, such as DIRECT \cite{Jones:1993} and CMA-ES~\cite{Hansen:2001:CDS:1108839.1108843}; it is only based on the GP model of the blackbox function $f$ and does not require additional evaluations of $f$. 

The Bayesian optimization procedure is shown in Algorithm~\ref{alg:bo}. 
\begin{algorithm}
\caption{Bayesian Optimization}
\label{alg:bo}
\begin{algorithmic}[1]
{
\STATE Initialize $\mathcal{D}_{0}$ as $\emptyset$.
\FOR{$t=1,2,\dots$}
  \STATE Find $\vx_{t+1}\! \in\! \mathbb{R}^D$ by optimizing the acquisition function $u$: $\vx_{t+1} \!= \!\argmax_{\vx \in {\cal X}} u(\vx|\mathcal{D}_t).$ 
  \STATE Augment the data $\mathcal{D}_{t+1} = \mathcal{D}_{t} \cup \{(\vx_{t+1}, f(\vx_{t+1}))\}$.
  \STATE Update the kernel hyper-parameters.
\ENDFOR
}
\end{algorithmic}
\end{algorithm}

\section{Random Embedding for Bayesian Optimization}\label{sec:rembo}
Before introducing our new algorithm and its theoretical properties, we need to define what we mean by effective dimensionality formally.
\begin{mydefinition}\label{def:effdim}
A function $f: \mathbb{R}^{D} \rightarrow \mathbb{R}$ is said to have \textbf{effective dimensionality} $d_e$, with $d_e \leq D$, if 
\begin{itemize}
\denselist
	\item there exists a linear subspace ${\cal T}$ of dimension $d_e$ such that for all $\vx_{\top} \in {\cal T} \subset \mathbb{R}^D$ and $\vx_{\bot} \in {\cal T}^{\bot} \subset \mathbb{R}^D$, we have $f(\vx_{\top} + \vx_{\bot}) = f(\vx_{\top})$, where ${\cal T}^{\bot}$ denotes the orthogonal complement of ${\cal T}$; and
	\item $d_e$ is the smallest integer with this property.
\end{itemize}
We call ${\cal T}$ the \textbf{effective subspace} of $f$ and ${\cal T}^{\bot}$ the \textbf{constant subspace}.
\end{mydefinition}
\vspace{-2mm}
This definition simply states that the function does not change along the coordinates $\vx_{\bot}$, and this is why we refer
to ${\cal T}^{\bot}$ as the {constant subspace}.
Given this definition,
the following theorem shows that problems of low effective dimensionality can be solved via random embedding.

\begin{theorem}
\label{prop:1}
Assume we are given a function $f: \mathbb{R}^{D} \rightarrow \mathbb{R}$ with effective dimensionality $d_e$ and a random matrix $\vA \in \mathbb{R}^{D\times d}$ with independent entries sampled according to $\mathcal{N}(0, 1)$ and $d\geq d_e$. Then, with probability 1, for any $\vx \in \mathbb{R}^D$, there exists a $ \vy \in \mathbb{R}^d$ such that $f(\vx) = f(\vA\vy)$.
\end{theorem}
\vspace{-5mm}
\begin{proof}
Please refer to the appendix.
\end{proof}

Theorem~\ref{prop:1} says that given any $\vx \in \mathbb{R}^D$ and a random matrix $\vA \in \mathbb{R}^{D\times d}$, with probability $1$, there is a point $\vy \in \mathbb{R}^d$ such that $f(\vx) = f(\vA\vy)$. 
This implies that for any optimizer $\vx^{\star} \in \mathbb{R}^D$, there is a point $\vy^{\star} \in \mathbb{R}^d$ with $f(\vx^\star) = f(\vA\vy^{\star})$. Therefore, instead of optimizing in the high dimensional space, we can optimize the function $g(\vy) = f(\vA\vy)$ in the lower dimensional space.
This observation gives rise to our new Random EMbedding Bayesian Optimization (REMBO) algorithm (see Algorithm \ref{alg:embed}). REMBO first draws a random embedding (given by $\vA$) and then performs Bayesian optimization in this embedded space.

\begin{algorithm}
\caption{REMBO: Bayesian Optimization with Random Embedding. Blue text denotes parts that are changed compared to standard Bayesian Optimization.}
\label{alg:embed}
\begin{algorithmic}[1]
{
\STATE \textcolor{blue}{Generate a random matrix $\vA \in \mathbb{R}^{D\times d}$}
\STATE \textcolor{blue}{Choose the bounded region set $\mathcal{Y} \subset \mathbb{R}^{d}$}
\STATE Initialize $\mathcal{D}_{0}$ as $\emptyset$.
\FOR{$t=1,2,\dots$}
  \STATE Find \textcolor{blue}{$\vy_{t+1} \in \mathbb{R}^d$} by optimizing the acquisition function $u$: \textcolor{blue}{$\vy_{t+1} = \argmax_{\vy\in \mathcal{Y}} u(\vy|\mathcal{D}_t).$} 
  \STATE Augment the data $\mathcal{D}_{t+1} = \mathcal{D}_{t} \cup \{ \textcolor{blue}{(\vy_{t+1}, f(\vA\vy_{t+1}))} \}$.
\STATE Update the kernel hyper-parameters.
\ENDFOR
}
\end{algorithmic}
\end{algorithm}

\begin{figure}[t]
\centering
  \includegraphics[scale=0.35]{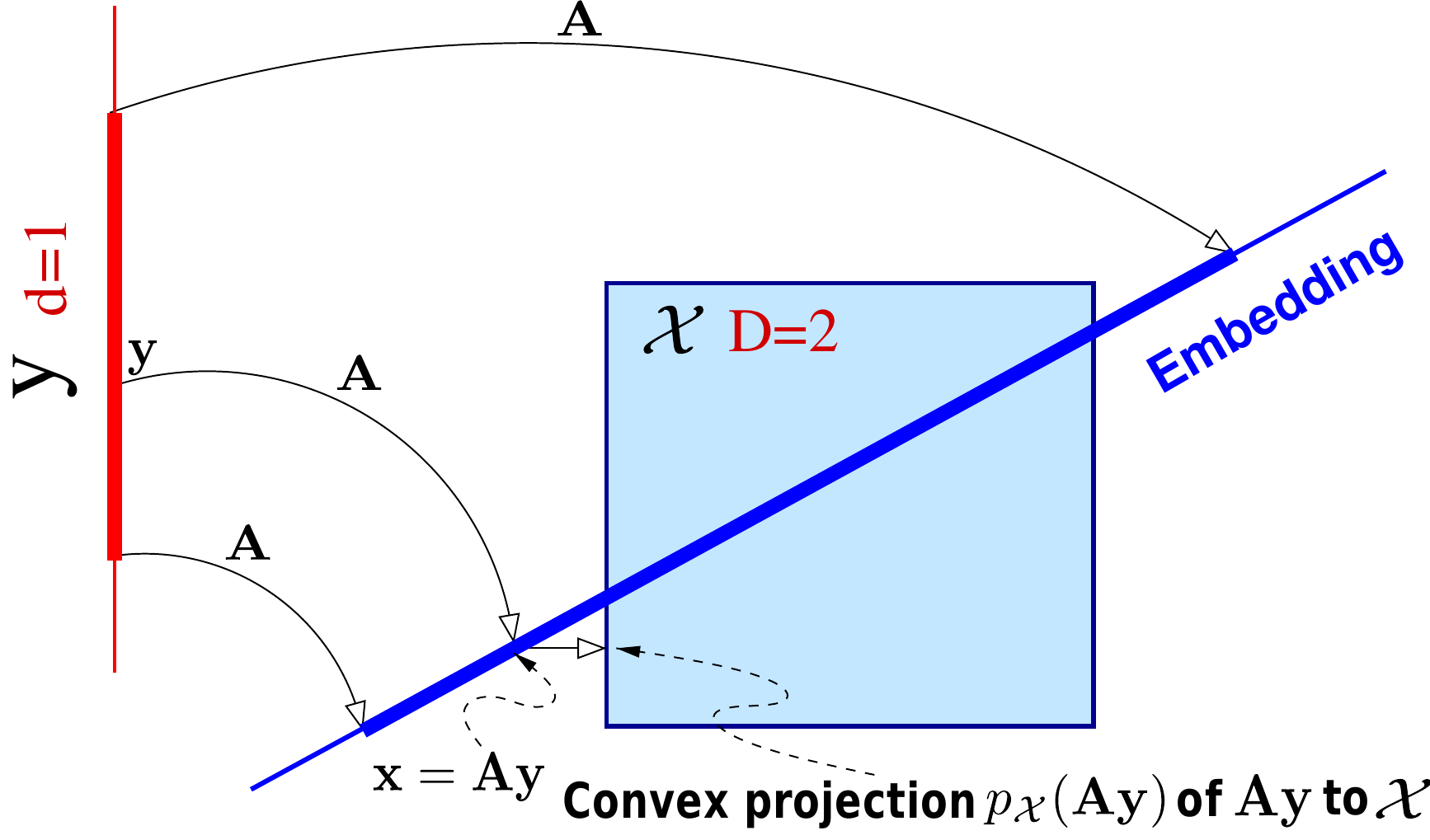}
  \caption{Embedding from $d=1$ into $D=2$. The box illustrates the 2D constrained space ${\cal X}$, while the thicker red line illustrates the 1D constrained space $\mathcal{Y}$. Note that if $\vA\vy$ is outside $\mathcal{X}$, it is projected onto $\mathcal{X}$. The set $\mathcal{Y}$ must be chosen large enough so that the projection of its image, $\vA \mathcal{Y}$, onto the effective subspace (vertical axis in this diagram) covers the vertical side of the box.}
  \label{fig:proj}
\end{figure}

In many practical optimization tasks, the goal is to optimize $f$ over a compact subset $\mathcal{X} \subset \mathbb{R}^{D}$ (typically a box), and $f$ can often not be evaluated outside of $\mathcal{X}$. Therefore, when REMBO selects a point $\vy$ such that $\vA\vy$ is outside the box $\mathcal{X}$, it projects $\vA\vy$ onto $\mathcal{X}$ before evaluating $f$. That is, 
$g(\vy) = f(p_{\mathcal{X}}(\vA\vy))$, where $p_{\mathcal{X}}:\mathbb{R}^D \rightarrow \mathbb{R}^D$ is the standard projection operator for our box-constraint: $p_{\mathcal{X}}(\vy) = {\arg \min}_{\vz\in \mathcal{X}} \|\vz-\vy\|_2$; see Figure~\ref{fig:proj}.
We still need to describe how REMBO chooses the bounded region $\mathcal{Y} \subset \mathbb{R}^{d}$, inside which it performs Bayesian optimization. This is important because REMBO's effectiveness depends on the size of $\mathcal{Y}$. Locating the optimum within $\mathcal{Y}$ is easier if $\mathcal{Y}$ is small, but if we set $\mathcal{Y}$ too small it may not actually contain the global optimizer.
In the following theorem, we show that we can choose $\mathcal{Y}$ in a way that only depends on the effective dimensionality $d_e$ such that the optimizer of the original problem is contained in the low dimensional space with constant probability.

\begin{theorem}
\label{prop:2}
Suppose we want to optimize a function $f: \mathbb{R}^{D} \rightarrow \mathbb{R}$ with effective dimension $d_e \leq d$ subject to the box constraint $\mathcal{X} \subset \mathbb{R}^D$, where $\mathcal{X}$ is centered around $\mathbf{0}$. 
Suppose further that the effective subspace $\cal T$ of $f$ is such that $\cal T$ is the span of $d_e$ basis vectors, and let $\vx^{\star}_\top \in \cal{T} \cap \mathcal{X}$ be an optimizer of $f$ inside $\mathcal{T}$. 
If $\vA$ is a $D\times d$ random matrix with independent standard Gaussian entries,
there exists an optimizer $\vy^\star \in \mathbb{R}^{d}$ such that $f(\vA\vy^\star) = f(\vx^\star_\top)$ and $\|\vy^\star\|_2 \leq \frac{\sqrt{d_e}}{\epsilon}\|\vx^{\star}_\top\|_2$ with probability at least $1-\epsilon$.
\end{theorem}
\vspace{-6mm}
\begin{proof}
Please refer to the appendix.
\end{proof}

Theorem~\ref{prop:2} says that if the set $\mathcal{X}$ in the original space is a box constraint, then there exists an optimizer $\vx^\star_\top \in \mathcal{X}$ that is $d_e$-sparse such that with probability at least $1-\epsilon$, $\|\vy^\star\|_2 \leq \frac{\sqrt{d_e}}{\epsilon}\|\vx^{\star}_\top\|_2$ where $f(\vA\vy^\star) = f(\vx^\star_\top)$. If the box constraint is $\mathcal{X} = [-1,1]^D$ (which is always achievable through rescaling), we have with probability at least $1-\epsilon$ that 
\[ \|\vy^\star\|_2 \leq \frac{\sqrt{d_e}}{\epsilon}\|\vx^{\star}_\top\|_2 \leq \frac{\sqrt{d_e}}{\epsilon}\sqrt{d_e}. \]
Hence, to choose $\mathcal{Y}$, we must ensure that the ball of radius $d_e/\epsilon$, centred at the origin, lies inside $\mathcal{Y}$. 

In practice, we have found that it is very unlikely that the optimizer falls on the corner of the box constraint,
implying that $\|x_{\top}^{\star}\| < \sqrt{d_e}$. Thus setting $\calY$ too big may
be unnecessarily wasteful. To improve our understanding of this effect, we developed a simulation study, in which we drew 
random Gaussian matrices, used them to map various potential optimizers $\vx^{\star}_\top$ to their corresponding points $\vy^{\star}_\top \in \calY$,
and studied the norms of $\vy^{\star}_\top$.

Assume for simplicity of presentation that $\calY$ is axis-aligned and $d_e$-dimensional (the argument applies when $d > d_e$). 
The section of the random matrix $\vA$ that maps points in $\calY$ to $\calT$ is a 
random Gaussian matrix of dimension $d_e\times d_e$. Let us call this section of 
the matrix $\vB$.
Since random Gaussian matrices are rotationally invariant in distribution, we have for any orthonormal matrix 
$\vO$ and a random Gaussian matrix $\vB$, $\vO\vB \,{\buildrel d \over =}\, \vB$. That is, $\vO\vB$ and $\vB$ are equal in distribution.
Similarly, for $\vB^{-1}$, 
$\vO\vB^{-1} = \left(\vB \vO^T\right)^{-1} \,{\buildrel d \over =}\, \vB^{-1}$. 
Therefore, $\vB^{-1}$
is also rotationally invariant. Hence, 
$\|\vB^{-1}\vx_{\top}\|_{\infty} \,{\buildrel d \over =}\, \|\vB^{-1}\vx'_{\top}\|_{\infty}$
as long as $\|\vx_{\top}\|_2 = \|\vx'_{\top}\|_2$. Following this equivalence for the supremum norm of projected vectors, it suffices to choose a point 
with the largest norm in $[-1, 1]^{d_e}$ in our simulations. We chose $\vx_{\top} = [1, 1, \cdots, 1]$.
 
We conducted simulations for several embedding dimensions, 
$d_e \in \{1, 2, \cdots, 50\}$, by drawing $10000$ random Gaussian
matrices and computing $\|\vB^{-1}\vx\|_{\infty}$.
We found that with empirical probability above $1-\epsilon$ (for decreasing values of $\epsilon$), it was the case that
$$\|\vB^{-1}\vx\|_{\infty} < \frac{1}{\epsilon} \max\{\log(d_e), 1\}.$$
These simulations indicate that we could set $\mathcal{Y} = \left[-\frac{1}{\epsilon} \max\{\log(d_e), 1\}, 
\frac{1}{\epsilon} \max\{\log(d_e), 1\}\right]^{d_e}$. We did this in our experiments and in particular chose $\epsilon =\log(d)/\sqrt{d}$, so that $\calY$ was $[-\sqrt{d}, \sqrt{d}]^d$. Note that Theorem~\ref{prop:2} is not useful for this choice, which suggests that there is room to improve this aspect of our theory.

Some careful readers may wonder about the effect of the extrinsic dimensionality $D$. 
In the following theorem, we show that given the same intrinsic dimensions, 
the extrinsic dimensionality does not have an effect at all; in other words, REMBO is invariant to 
the addition of unimportant dimensions. 
\begin{theorem}[Invariance to addition of unimportant dimensions]
\label{prop:invar-D}
	Let $f: \mathbb{R}^{d_e} \rightarrow \mathbb{R}$ and for any $D \in \mathbb{N}$, $D \geq d_e$, define $f_{D}: \mathbb{R}^{D} \rightarrow \mathbb{R}$
	such that	$f_D$ adds $D-d_e$ truly unimportant dimensions to $f$: $f_{D}(\vz) = f(\vz_{1:d_e})$.
	Let $\vA_1 \in \mathbb{R}^{D_1\times d}$ and $\vA_0 \in \mathbb{R}^{(D_2-D_1)\times d}$ be random Gaussian matrices with $D_2 \geq D_1 \geq d$ and let
	$\vA_2 = 
	\begin{bmatrix}
	\vA_1\\
  \vA_0
	\end{bmatrix}.$
	Then, REMBO run using the same dimension $d \ge d_e$ and bounded region $\mathcal{Y}$
	yields exactly the same function values when run with $\vA_1$ on $f_{D_1}$ 
	as when run with $\vA_2$ on $f_{D_2}$.
	\end{theorem}
\begin{proof}
	We only need to show that for each $\vy \in \mathbb{R}^{d}$, we have $f_{D_1}( \vA_1 \vy ) = f_{D_2}( \vA_2 \vy )$
	since this step of REMBO (line 6 of Algorithm \ref{alg:embed}) is the only one that differs between the two algorithm runs. 
	When this function evaluation step yields the same results for every $\vy \in \mathbb{R}^{d}$, then
	the two REMBO runs behave identically since the algorithm is otherwise identical and deterministic after the selection of $\vA$ in Step 1.
	Since 
	$\vA_2 = 
	\begin{bmatrix}
	\vA_1\\
  \vA_0
	\end{bmatrix}$, we have 
  $\vA_2 \vy = 
	\begin{bmatrix}
	\vA_1 \vy\\
  \vA_0 \vy
	\end{bmatrix}$.
	Since $D_2 \ge D_1 \ge d_e$, the first $d_e$ entries of this $D_2 \times 1$ vector $\vA_2 \vy$ are the first $d_e$ entries of $\vA_1 \vy$.
	We thus have 
	$f_{D_1}( \vA_1 \vy ) = f([\vA_1 \vy]_{1:d_e}) = f([\vA_2 \vy]_{1:d_e}) = f_{D_2}( \vA_2 \vy )$.
\end{proof}

Finally, we show that REMBO is also invariant to rotations in the sense that
given different rotation matrices, running REMBO would result in the same distributions of observed function values. The argument is made concise in the following results.
\begin{lemma}
\label{lem:invar-rot}
	Consider function $f: \mathbb{R}^{D} \rightarrow \mathbb{R}$. 
	Let $f_{\vR}: \mathbb{R}^{D} \rightarrow \mathbb{R}$ be such that 
	$f_{\vR}(\vx) = f(\vR \vx)$ for some orthonormal matrix $\vR \in \mathbb{R}^{D \times D}$.
	Then, REMBO run in bounded region $\mathcal{Y}$
	yields exactly the same sequence of function values when run with $\vA$ on $f$ 
	as when run with $\vR^{-1} \vA$ on $f_{\vR}$ 
	for a matrix $\vA \in \mathbb{R}^{D\times d}$.
	\end{lemma}
\begin{proof}
	REMBO uses $f$ and $\vA$ (resp.\ $f_{\vR}$ and $\vR^{-1} \vA$) only in one spot (in line 6).
	Thus, the proof is trivial by showing that $f(\vA \vy_{t+1}) = f_{\vR}(\vR^{-1} \vA \vy_{t+1})$ through simple algebra:
	\[f_{\vR}(\vR^{-1} \vA \vy_{t+1}) = f(\vR \vR^{-1} \vA \vy_{t+1}) = f(\vA \vy_{t+1}).\]	
\end{proof}
\begin{theorem}[Invariance to rotations]
\label{prop:invar}
	Consider function $f: \mathbb{R}^{D} \rightarrow \mathbb{R}$. 
	Let $f_{\vR}: \mathbb{R}^{D} \rightarrow \mathbb{R}$ be such that 
	$f_{\vR}(\vx) = f(\vR \vx)$ for some orthonormal matrix $\vR \in \mathbb{R}^{D \times D}$.
	Then, given random Gaussian matrices $\vA_1 \in \mathbb{R}^{D\times d}$ 
	and $\vA_2 \in \mathbb{R}^{D\times d}$,
	REMBO run in bounded region $\mathcal{Y}$
	yields in distribution the same sequence of function values when run with $\vA_1$ on $f$ 
	as when run with $\vA_2$ on $f_{\vR}$.
\end{theorem}
\begin{proof}
	Since $\vR$ is orthonormal, we have $\vR^{-1} \vA_1 \overset{d}{=} \vA_2$.
	Therefore, REMBO run in bounded region $\mathcal{Y}$
	yields in distribution the same sequence of function values when run with $\vR^{-1} \vA_1$ on $f_{\vR}$ 
	as when run with $\vA_2$ on $f_{\vR}$.
	We have also by Lemma~\ref{lem:invar-rot} that REMBO run in bounded region $\mathcal{Y}$
	yields exactly the same sequence of function values when run with $\vA_1$ on $f$ 
	as when run with $\vR^{-1} \vA_1$ on $f_{\vR}$. The conclusion follows from combining the previous arguments.
\end{proof}

\subsection{Increasing the Success Rate of REMBO}\label{sec:increasing_rembo_success}

Theorem~\ref{prop:2} only guarantees that $\mathcal{Y}$ contains the optimum with probability at least $1-\epsilon$;
with probability $\delta \le \epsilon$ the optimizer lies outside of $\mathcal{Y}$. There are several ways to guard against this problem. One is to simply run REMBO multiple times with different independently drawn random embeddings. 
Since the probability of failure with each embedding is $\delta$, 
the probability of the optimizer not being included in the considered space of $k$ independently drawn embeddings is $\delta^k$. 
Thus, the failure probability vanishes exponentially quickly in the number of REMBO runs, $k$. 
Note also that these independent runs can be trivially parallelized to harness the power of
modern multi-core machines and large compute clusters.




Another way of increasing REMBO's success rate is to increase the dimensionality $d$ it uses internally. When $d > d_e$, with probability $1$ we have ${ d \choose d_e }$ different embeddings of dimensionality $d_e$. 
That is,
we only need to select $d_e$ columns of $\vA \in \mathbb{R}^{D \times d}$ to represent the $d_e$ relevant dimensions of $\vx$.
The algorithm can achieve this by setting the remaining $d-d_e$ sub-components of the $d$-dimensional vector $\vy$ to zero. 
Informally, since we have more embeddings, it is more likely that one of these will include the optimizer.
In our experiments, we will assess the merits and shortcomings of these two strategies. 


\subsection{Choice of Kernel}\label{sec:choice_of_kernel}

Since REMBO uses GP-based Bayesian optimization to search in the region $\mathcal{Y} \subset \mathbb{R}^{d}$, 
we need to define a kernel between two points $\vy^{(1)}, \vy^{(2)} \in {\cal Y}$. 
We begin with the standard definition of the squared exponential kernel:
\begin{mydefinition} 
Let $K_{SE}(\vy) = \exp(-\|\vy\|^2/2)$.
Given a length scale $\ell > 0$, we define the corresponding {squared exponential} kernel as 
\[ k_{\ell}^d(\vy^{(1)},\vy^{(2)}) = K_{SE}\left(\frac{\vy^{(1)}-\vy^{(2)}}{\ell}\right) \]
\label{def:sekernel}
\end{mydefinition}
\vspace{-3mm}
It is possible to work with two variants of this kernel. 
First, we can use $k_{\ell}^d(\vy^1, \vy^2)$ as in Definition~\ref{def:sekernel}. We refer to this kernel as the low-dimensional kernel. We can also adopt an implicitly defined high-dimensional kernel on 
$\mathcal{X}$:
$$k_{\ell}^D(\vy^{(1)}, \vy^{(2)}) = K_{SE}\left( \frac{p_{\mathcal{X}}(\vA\vy^{(1)}) - p_{\mathcal{X}}(\vA\vy^{(2)})}{\ell} \right),$$
where $p_{\mathcal{X}}:\mathbb{R}^D \rightarrow \mathbb{R}^D$ is the projection operator for our box-constraint as above (see Figure~\ref{fig:proj}).

Note that when using this high-dimensional kernel, we are fitting the GP in $D$ dimensions. However, the search space is no longer the box $\mathcal{X}$, but it is instead given by the much smaller subspace $\{p_{\mathcal{X}}(\vA\vy): \vy \in {\cal Y} \}$. Importantly, in practice it is easier to maximize the acquisition function in this subspace.   

Both kernel choices have strengths and weaknesses. 
The low-dimensional kernel has the benefit of only requiring the construction of a GP in the space of intrinsic dimensionality $d$, whereas the high-dimensional kernel requires the GP to be constructed in a space of extrinsic dimensionality $D$.
However, the low-dimensional kernel may waste time exploring in the region of the embedding outside of ${\cal X}$ (see Figure 2) because two points far apart in this region may be projected via $p_{\mathcal{X}}$ to nearby points on the boundary of ${\cal X}$. The high-dimensional kernel is not affected by this problem because the search is conducted directly on $\{p_{\mathcal{X}}(\vA\vy): \vy \in {\cal Y} \}$ with distances calculated in $\cal X$ and not in $\cal Y$.

The choice of kernel also depends on whether our variables are continuous, integer or categorical. The categorical case is important because we often encounter optimization problems that contain discrete choices. We define our kernel for categorical variables as:
$$k^D_{\lambda}(\vy^{(1)}, \vy^{(2)}) = \exp\left(-\frac{\lambda}{2} h(s(\vA\vy^{(1)}), s(\vA\vy^{(2)}))^2 \right),$$
where $\vy^{(1)}, \vy^{(2)} \in {\cal Y} \subset \mathbb{R}^d$,
the function $s$ maps continuous $d$-dimensional vectors to discrete $D$-dimensional vectors, and 
$h$ defines the distance between two discrete vectors. In more detail, $s(\vx)$ first uses $p_{\mathcal{X}}$ to  project $\vx$ to $\bar{\vx} \in [-1,1]^D$. For each dimension $\bar{x}_i$ of $\bar{\vx}$, $s$ then maps $\bar{x}_i$ to a discrete value by scaling and rounding. 
In our experiments, following~\citeA{Hutter:2009}, we defined $h(\vx^{(1)}, \vx^{(2)}) = |\{i: x^{(1)}_i \neq x^{(2)}_i\}|$ 
so as not to impose an artificial ordering between the values of categorical parameters. 
In essence, we measure the distance between two points in the low-dimensional space as the Hamming distance between their mappings in the high-dimensional space.



\subsection{Hyper-parameter Optimization}\label{sec:hyperparameter_opt}
\begin{algorithm}[t]
\caption{Bayesian Optimization with Hyper-parameter Optimization.}
\label{alg:bohyper}
\begin{algorithmic}[1]
{
\INPUT Threshold $t_{\sigma}$.  
\INPUT Upper and lower bounds $U > L > 0$ for hyper-parameter. 
\INPUT Initial length scale hyper-parameter $\ell \in [L, U]$.
\STATE Initialize $C = 0$

\FOR{$t=1,2,\dots$}
  \STATE Find $\vx_{t+1}$ by optimizing the acquisition function $u$: $\vx_{t+1} = \argmax_{\vx \in {\cal X}} u(\vx|\mathcal{D}_t).$ 
  \IF{$\sqrt{\sigma^2(\vx_{t+1})} < t_{\sigma} $}
    \STATE $C = C + 1$
  \ELSE
    \STATE $C = 0$
  \ENDIF
  \STATE Augment the data $\mathcal{D}_{t+1} = \{\mathcal{D}_{t}, (\vx_{t+1}, f(\vx_{t+1})) \}$

  \IF{$t \mod 20 = 0$ \OR $C=5$}
    \IF{$C=5$}
        \STATE $U = \max\{0.9 \ell, L\}$
        \STATE $C = 0$
    \ENDIF
    \STATE Learn the hyper-parameter by optimizing the log marginal likelihood by using DIRECT and CMA-ES: $\ell=\argmax_{l \in [L, U]} \log p (\vf_{1:t+1}|\vx_{1:t+1}, l)$
  \ENDIF
\ENDFOR
}
\end{algorithmic}
\end{algorithm}

For Bayesian optimization (and therefore REMBO), it is difficult to manually estimate the true length scale hyper-parameter of a problem at hand. 
 To avoid any manual steps and to achieve robust performance across diverse sets of objective functions, in this paper we adopted an adaptive hyper-parameter optimization scheme. 
The length scale of GPs is often set by maximizing marginal likelihood~\cite{Rasmussen:2006,Jones:1998}. However, as demonstrated by \citeA{Bull:2011}, this approach, when implemented naively, may not guarantee convergence. This is not only true of approaches that maximize the marginal likelihood, but also of approaches that rely on Monte Carlo sampling from the posterior distribution \cite{Brochu:2010,Snoek:2012} when the number of data is very small, unless the prior is very informative.

Here, we propose to optimize the length scale parameter $\ell$ by maximizing the marginal likelihood subject to an upper bound $U$ which is decreased when the algorithm starts exploiting too much. 
Full details are given in Algorithm~\ref{alg:bohyper}.
We say that the algorithm is exploiting when the standard deviation at the maximizer of the acquisition function $\sqrt{\sigma(\vx_{t+1})}$ is less than some threshold $t_{\sigma}$ for $5$ consecutive iterations. 
Intuitively, this means that the algorithm did not emphasize exploration (searching in new parts of the space, where the predictive uncertainty is high) for $5$ consecutive iterations. When this criterion is met, the algorithm decreases its upper bound $U$ multiplicatively and re-optimizes the hyper-parameter subject to the new bound.
Even when the criterion is not met the hyper-parameter is re-optimized every $20$ iterations.
For each optimization of the acquisition function, the algorithm runs both DIRECT~\cite{Jones:1993} and CMA-ES~\cite{Hansen:2001:CDS:1108839.1108843} 
and uses the result of the best of the two options.
The astute reader may wonder about the difficulty of optimizing the 
acquisition functions. For REMBO, however, we have not found the optimization of the acquisition function
to be a problem since we only need to optimize it in the low-dimensional space
and our acquisition function evaluations are cheap, allowing us tens of
thousands of evaluations in seconds that (empirically) suffice to cover the 
low-dimensional space well.

The motivation of this algorithm is to rather err on the side of having too small a length scale:
given a squared exponential kernel $k_\ell$, with a smaller length scale than another kernel $k$, one can show that any function $f$ in the RKHS characterized by $k$ is also an element of the RKHS characterized by $k_\ell$.
Thus, when running expected improvement, one can safely use $k_\ell$ instead of $k$ as the kernel of the GP and still preserve convergence~\cite{Bull:2011}. We argue that (with a small enough lower bound $L$) the algorithm would eventually reduce the upper bound enough to allow convergence.
Also, the algorithm would not explore indefinitely as $L$ is required to be positive.
In our experiments, we set the initial constraint $[L, U]$ to be $[0.01, 50]$ and set $t_{\sigma} = 0.002$. 

We want to stress the fact that the above argument
is only known to hold for a class of kernels over continuous domains (e.g.\ squared exponential and Mat\'ern class kernels). 
Although we believe that a similar argument could be made for
integer and categorical kernels, 
rigorous arguments concerning convergence under these kernels remain a challenge
in Bayesian optimization.


\section{Experiments}\label{sec:experiments}

We now study REMBO empirically. We first use synthetic functions of small intrinsic dimensionality $d_e=2$ but 
extrinsic dimension $D$ up to $1$ billion to demonstrate REMBO's independence of $D$.
Then, we apply REMBO to automatically optimize the 47 parameters of a widely-used mixed integer linear programming solver and demonstrate that it achieves state-of-the-art performance.
However, we also warn against the blind application of REMBO. To illustrate this, we study REMBO's performance for tuning the 14 parameters of a random forest body part classifier used by Kinect. In this application, all the $D=14$ parameters appear to be important, and while REMBO (based on $d=3$) finds reasonable solutions (better than random search and comparable to what domain experts achieve), standard Bayesian optimization can outperform REMBO (and the domain experts) in such moderate-dimensional spaces. More optimistically, this random forest tuning application shows that REMBO does not fail catastrophically when it is not clear that the optimization problem has low effective dimensionality.

\subsection{Experimental Setup}\label{sec:experimental-setup}

For all our experiments, we used a single robust version of REMBO that automatically sets its GP's length scale parameter as described in 
Section~\ref{sec:hyperparameter_opt}. 
The code for REMBO, as well as all data used in our experiments is publicly available at \url{https://github.com/ziyuw/rembo}.


Some of our experiments required substantial computational resources, with the computational 
expense of each experiment depending mostly on the cost of evaluating the respective black-box function. 
While the synthetic experiments in Section \ref{exp:billion} only required minutes for each run of each method, optimizing the mixed integer programming solver in Section \ref{exp:mip_opt} required 4-5 hours per run, and optimizing the random forest classifier in Section \ref{exp:rf_configuration} required 4-5 days per run. In total, we used over half a year of CPU time for the experiments in this paper. 
In the first two experiments, we study the effect of our two methods for increasing REMBO's success rate (see Section \ref{sec:increasing_rembo_success}) by running different numbers of independent REMBO runs with different settings of its internal dimensionality $d$.



%

\subsection{Bayesian Optimization in a Billion Dimensions}\label{exp:billion}


\begin{figure*}[t!]
\centering
  \includegraphics[scale=0.5]{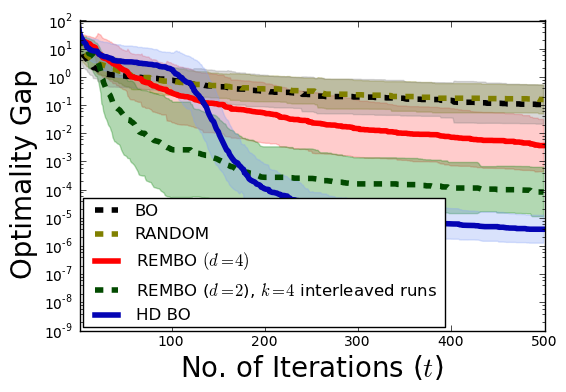}
  \includegraphics[scale=0.5]{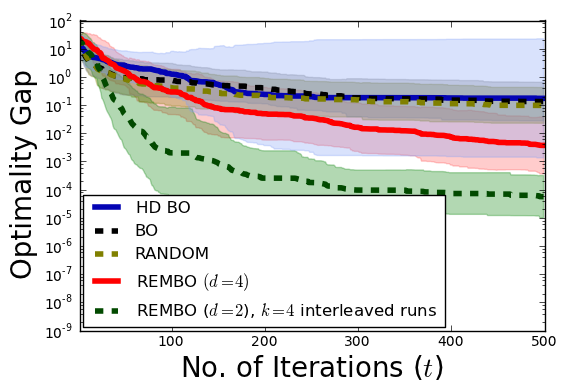}
  \includegraphics[scale=0.5]{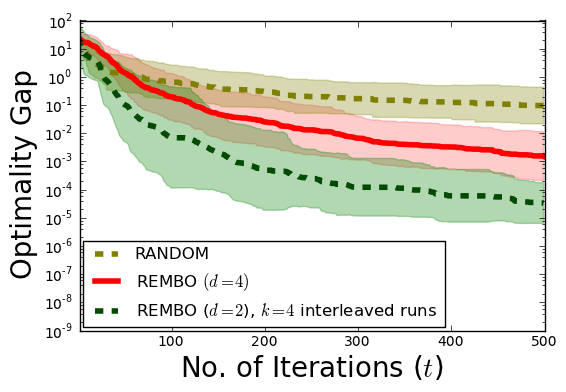}
  \caption{Comparison of random search (RANDOM), Bayesian optimization (BO),
  method by~\protect\citeA{Chen:2012} (HD BO), and REMBO.
Left: $D=25$ extrinsic dimensions; Right: $D=25$, with a rotated objective function; Bottom: $D=10^9$ extrinsic dimensions. We plot means and $1/4$ standard deviation confidence intervals of the optimality gap across 50 trials.}
  \label{fig:standard}
\end{figure*}

The experiments in this section employ a standard $d_e=2$-dimensional benchmark function for Bayesian optimization, embedded in a $D$-dimensional space. That is, we add $D-2$ additional dimensions which do not affect the function at all.
More precisely, the function whose optimum we seek is $f(\vx_{1:D}) = g(x_i,x_j)$, where $g$ is the Branin function
\begin{eqnarray*}
g(x_1, x_2) 
= (x_2- \frac{5.1}{4\pi^2}x_1^2 + \frac{5}{\pi}x_1 - 6)^2 + 
10(1-\frac{1}{8\pi})\cos(x_1) + 10
\end{eqnarray*}
and where $i$ and $j$ are selected once using a random permutation.
To measure the performance of each optimization method, we used the \emph{optimality gap}: the difference of the best function value it found and the optimal function value.   

\begin{table}[h]
\small
\begin{center}
\begin{tabular}{ l |  c  c  c}
\hline
$k$ & $d=2$ & $d=4$ & $d=6$ \\
\hline
  10& 0.0022 $\pm$ 0.0035 & 0.1553 $\pm$ 0.1601 & 0.4865 $\pm$  0.4769 \\
  5 & 0.0004 $\pm$ 0.0011 & 0.0908 $\pm$ 0.1252 & 0.2586 $\pm$ 0.3702 \\
  4 & 0.0001 $\pm$ 0.0003 & 0.0654 $\pm$ 0.0877 & 0.3379 $\pm$ 0.3170 \\
  2 & 0.1514 $\pm$ 0.9154 & 0.0309 $\pm$ 0.0687 & 0.1643 $\pm$ 0.1877 \\
  1 & 0.7406 $\pm$ 1.8996 & 0.0143 $\pm$ 0.0406 & 0.1137 $\pm$ 0.1202 \\
\end{tabular}
\end{center}
\caption{Optimality gap for $d_e=2$-dimensional Branin function embedded in $D=25$ dimensions, for
REMBO variants using a total of $500$ function evaluations. The variants differed in the internal dimensionality $d$ and in the number of interleaved runs $k$ (each such run was only allowed $500/k$ function evaluations).
We show mean and standard deviations of the optimality gap achieved after 500 function evaluations. 
\label{tab:rembo_variants_results}}
\end{table}

We evaluate REMBO using a fixed budget of $500$ function
evaluations that is spread across multiple interleaved runs ---
for example, when using $k = 4$ interleaved REMBO runs,
each of them was only allowed $125$ function evaluations. 
We study the choices of $k$ and $d$ by considering several combinations of these values.
The results in Table \ref{tab:rembo_variants_results} demonstrate that interleaved runs helped improve REMBO's performance.
We note that in 13/50 REMBO runs, the global optimum was indeed not contained in the box $\mathcal{Y}$ REMBO searched with $d=2$; this is the reason for the poor mean performance of REMBO with $d=2$ and $k=1$.
However, the remaining $37$ runs performed very well, and REMBO thus performed well when using multiple interleaved runs: with a failure rate of 13/50=0.26 per independent run, the failure rate using $k=4$ interleaved runs is only $0.26^4\approx 0.005$. One could easily achieve an arbitrarily small failure rate by using many independent parallel runs. 
Using a larger $d$ is also effective in increasing the probability of the optimizer falling into REMBO's box $\mathcal{Y}$ but at the same time slows down REMBO's convergence (such that interleaving several short runs loses its effectiveness). 
 
Next, we compared REMBO to standard Bayesian optimization (BO) and to random search, for an extrinsic dimensionality of $D=25$. 
Standard BO is well known to perform well in low dimensions, but to degrade above a tipping point of about 15-20 dimensions. Our results for $D=25$ (see Figure \ref{fig:standard}, left) confirm that BO performed rather poorly just above this critical dimensionality (merely tying with random search). 
REMBO, on the other hand, still performed very well in 25 dimensions.

One important advantage of REMBO is that --- in contrast to the approach of \citeA{Chen:2012} --- it does not require the effective dimension to be coordinate aligned. To demonstrate this fact empirically, we rotated the embedded Branin function by an orthogonal rotation matrix $\vR \in \mathbb{R}^{D\times D}$. That is, we replaced $f(\vx)$ by $f(\vR \vx)$.
Figure~\ref{fig:standard} (middle) shows that REMBO's performance is not affected by this rotation. 

Finally, since REMBO is independent of the extrinsic dimensionality $D$ as long as the intrinsic dimensionality $d_e$ is small, it performed just as well in $D=1\,000\,000\,000$ dimensions (see Figure \ref{fig:standard}, right). 
To the best of our knowledge, the only other existing method that can be run in such high dimensionality is random search.

For reference, we also evaluated the method of \citeA{Chen:2012} for these functions, confirming that it does not handle rotation gracefully: 
while it performed best in the non-rotated case for $D=25$, it performed worst in the rotated case. It could not be used efficiently for more than $D=1,000$.
Based on a Mann-Whitney U test with Bonferroni multiple-test correction, all performance differences were statistically significant, except Random vs.\ standard BO.
Finally, comparing REMBO to the method of \citeA{Chen:2012}, we also note that REMBO is much simpler to implement and that its results are very reliable (with interleaved runs).

\subsection{Synthetic Discrete Experiment}
In this section, we test the high-dimensional kernel with a synthetic experiment. Specifically, we again optimize the Branin function, but restrict its domain to $225$ discrete points on a regular grid.
As above, we added $23$ additional irrelevant dimensions to make the problem 25-dimensional in total. 

We used a small fixed budget of $100$ function evaluations for all algorithms involved as the problem would require no more than $225$ evaluations to be solved completely. We used $k=4$ interleaved runs for REMBO.
We again compare REMBO to random search and standard BO. For REMBO, we use the high-dimensional kernel to handle the discrete nature of the problem.
The result of the comparison is summarized in Figure~\ref{fig:disbran}.
Standard BO again suffered from the high extrinsic dimensionality
and performed slightly worse than random search. REMBO, on the other hand, performed well in this setting.
\begin{figure*}[t!]
\centering
  \includegraphics[scale=0.5]{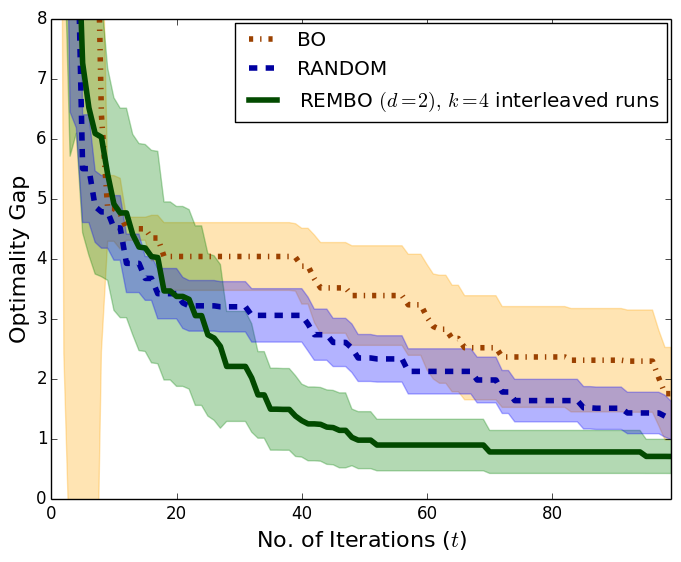}
  \caption{Comparison of random search (RANDOM), Bayesian optimization (BO), and REMBO. $D=25$ extrinsic dimensions. We plot means and $1/4$ standard deviation confidence intervals of the optimality gap across 50 trials.}
  \label{fig:disbran}
\end{figure*}

\subsection{Automatic Configuration of a Mixed Integer Linear Programming Solver}\label{exp:mip_opt}
State-of-the-art algorithms for solving hard computational problems tend to parameterize several design choices in order to allow a customization of the algorithm to new problem domains. Automated methods for algorithm configuration have recently demonstrated that substantial performance gains of state-of-the-art algorithms can be achieved in a fully automated fashion~\cite{Mockus:1999,ParamILS-JAIR,Hutter:2010,ValEtAl11,Bergstra:2011,Wang:2011}. 
These successes have led to a paradigm shift in algorithm development towards the 
active design of highly parameterized frameworks that can be automatically customized to particular problem domains using optimization~\cite{Hoos:2012:PO:2076450.2076469,Bergstra:model_search,Thornton:2013}.
The resulting algorithm configuration problems have been shown to have low dimensionality~\cite{Hutter:2014}, and here, we demonstrate that REMBO can exploit this low dimensionality even in the discrete spaces typically encountered in algorithm configuration. 
We use a configuration problem obtained from \citeA{Hutter:2010}, aiming to configure the 40 binary and 7 categorical parameters of \text{lpsolve}~\cite{lpsolve}
, a popular mixed integer programming (MIP) solver that has been downloaded over 40\,000 times in the last year.
The objective is to minimize the optimality gap \texttt{lpsolve} can obtain in a time limit of five seconds for a MIP encoding of a wildlife corridor problem from computational sustainability~\cite{ghs08:connection}.
Algorithm configuration usually aims to improve performance for a representative set of problem instances, and effective methods need to solve two orthogonal problems: searching the parameter space effectively and deciding how many instances to use in each evaluation (to trade off computational overhead and over-fitting). Our contribution is for the first of these problems; to focus on how effectively the different methods search the parameter space, we only consider configuration on a single problem instance.

Due to the discrete nature of this optimization problem, we could only apply REMBO using the high-dimensional kernel for categorical variables $k^D_{\lambda}(\vy^{(1)}, \vy^{(2)})$ described in Section \ref{sec:choice_of_kernel}. While we have not proven any theoretical guarantees for discrete optimization problems, REMBO appears to effectively exploit the low effective dimensionality of at least this particular optimization problem.

\begin{figure}[tb]
\begin{center}
  \includegraphics[scale=0.41]{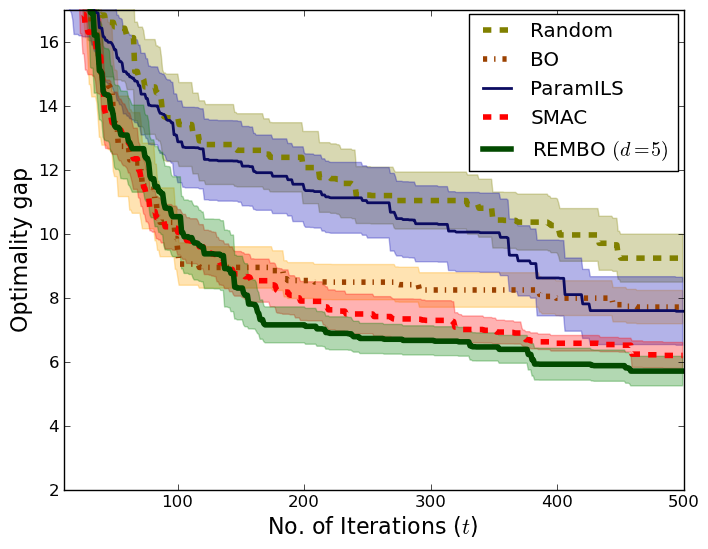}
  \includegraphics[scale=0.41]{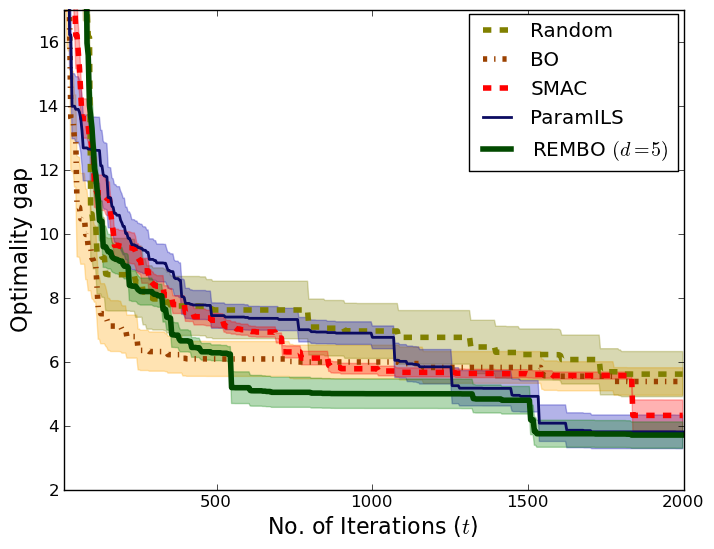}
  \caption{Performance of various methods for configuration of \texttt{lpsolve}; we show the optimality gap \texttt{lpsolve} achieved with the configurations found by the various methods (lower is better). Left: a single run of each method; Right: performance with $k=4$ interleaved runs.\label{fig:lpsolve}}
\end{center}
\end{figure}

Figure \ref{fig:lpsolve} (left) compares BO, REMBO, and the baseline random search against 
ParamILS \cite{ParamILS-JAIR} 
and SMAC \cite{Hutter:2011}. 
ParamILS and SMAC were specifically designed for the configuration of algorithms with many discrete parameters and define the current state of the art for this problem. Nevertheless, here SMAC and our vanilla REMBO method performed best. Based on a Mann-Whitney U test with Bonferroni multiple-test correction, they both yielded statistically significantly better results than both Random and standard BO; no other performance differences were significant.
The figure only shows REMBO with $d=5$ to avoid clutter, but we did not optimize this parameter; the only other value we tried ($d=3$) resulted in indistinguishable . 

As in the synthetic experiment, REMBO's performance could be further improved by using multiple interleaved runs.
However, as shown by~\citeA{HutHooLey12-ParallelAC}, multiple independent runs can also improve the performance of SMAC and especially ParamILS. 
Thus, to be fair, we re-evaluated all approaches using interleaved runs. Figure \ref{fig:lpsolve} (right) shows that ParamILS and REMBO benefitted most from interleaving $k=4$ runs. However, the statistical test results did not change, still showing that SMAC and REMBO outperformed Random and BO, with no other significant performance differences.

\subsection{Automatic Configuration of Random Forest Kinect Body Part Classifier}\label{exp:rf_configuration}
We now evaluate REMBO's performance for optimizing the 14 parameters of a random forest body part classifier. This classifier closely follows the proprietary system used in the Microsoft Kinect \cite{Shotton:2011} and is available at \url{https://github.com/david-matheson/rftk}.

We begin by describing some details of the dataset and classifier in order to build intuition for the objective function and the parameters being optimized. The data we used consists of pairs of depth images and ground truth body part labels. Specifically, we used 1\,500 pairs of 320x240 resolution depth and body part images, each of which was synthesized from a random pose of the CMU mocap dataset. Depth, ground truth body parts and predicted body parts (as predicted by the classifier described below) are visualized for one pose in Figure \ref{fig:rf-all}~(left). There are 19 body parts plus one background class. 
For each of these 20 possible labels, the training data contained 25\,000 pixels, randomly selected from 500 training images. Both validation and test data contained \emph{all} pixels in the 500 validation and test images, respectively.

\begin{figure}[t]
\centering
  \raisebox{0.6cm}{\includegraphics[height=0.2\textheight]{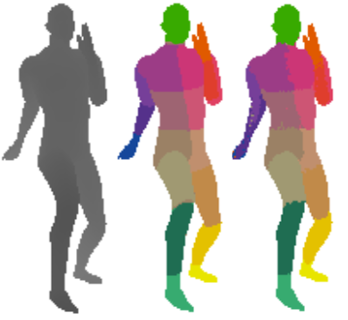}} 
  ~~~~~~~~~~
  \raisebox{0.6cm}{\includegraphics[height=0.2\textheight]{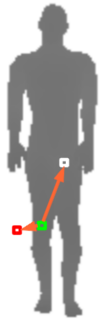}} 
  \caption{Left: ground truth depth, ground truth body parts and predicted body parts; Right: features specified by offsets u and v.}
  \label{fig:rf-all}
\end{figure}

The random forest classifier is applied to one pixel $P$ at a time. At each node of each of its decision trees, it computes the depth difference between two pixels described by offsets from $P$ and compares this to a threshold. 
At training time, many possible pairs of offsets are generated at random, and the pair yielding highest information gain for the training data points is selected. Figure \ref{fig:rf-all}~(right) visualizes a potential feature for the pixel in the green box: it computes the depth difference between the pixels in the red box and the white box, specified by respective offsets u and v.
At training time, u and v are drawn from two independent 2-dimensional Gaussian distributions, each of which is parameterized by its two mean parameters $\mu_1$ and $\mu_2$ and three covariance terms $\Sigma_{11}$, $\Sigma_{12}$, and $\Sigma_{22}$ ($\Sigma_{21}=\Sigma_{12}$ because of symmetry). 
These constitute 10 of the parameters that need to be optimized, with range [-50,50] for the mean components and [1, 200] for the covariance terms. Low covariance terms yield local features, while high terms yield global features. 
Next to these ten parameters, the random forest classifier has four other standard parameters, outlined in Table \ref{table:rf}. 
It is well known in computer vision that many of the parameters described here are important. Much research has been devoted to identifying their best values, but results are dataset specific, without definitive general answers.

\begin{table}[t]
\caption{Parameter ranges for random forest classifier. For the purpose of optimization, the maximum tree depth and the number of potential offsets were transformed to log space.}
\label{table:rf}
\begin{center}
{\footnotesize
\begin{tabular}{ll}
\multicolumn{1}{c}{\bf Parameter}  &\multicolumn{1}{c}{\bf Range}
\\ \hline \\
Max. tree depth         &[1 60] \\
Min. No. samples for non leaf nodes             & [1 100]  \\
No. potential offsets to evaluate           &[1 5000] \\
Bootstrap for per tree sampling           &[T F] \\
\end{tabular}
}
\end{center}
\end{table}


The objective in optimizing these RF classifier parameters is to find a parameter setting that learns the 
best classifier in a given time budget of five minutes. To enable competitive performance in this short amount of time, at each node of the tree only a random subset of data points is considered.
Also note that the above parameters do not include the number of trees $T$ in the random forest; since performance improves monotonically in $T$, we created as many trees as possible in the time budget. Trees are constructed depth first and returned in their current state when the time budget is exceeded. 
Using a fixed budget results in a subtle optimization problem because of the complex interactions between the various parameters (maximum depth, number of potential offsets, number of trees and accuracy). 

It is unclear a priori whether a low-dimensional subspace of these 14 interacting parameters exists that captures the classification accuracy of the resulting random forests. 
We performed large-scale computational experiments with REMBO, random search, and standard Bayesian optimization (BO) to study this question. In this experiment, we used the high-dimensional kernel for REMBO to avoid the potential over-exploration problems of the low-dimensional kernel described in Section~\ref{sec:choice_of_kernel}. We believed that $D=14$ dimensions would be small enough to avoid inefficiencies in fitting the GP in $D$ dimensions. This belief was confirmed by the observation that standard BO (which operates in $D=14$ dimensions) performed well for this problem. 

\begin{figure}[t]
\centering
   \includegraphics[scale=0.41]{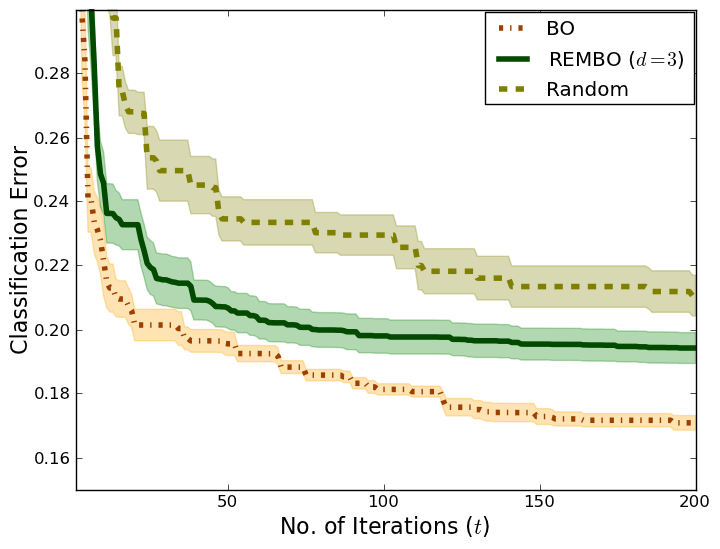}
   \includegraphics[scale=0.41]{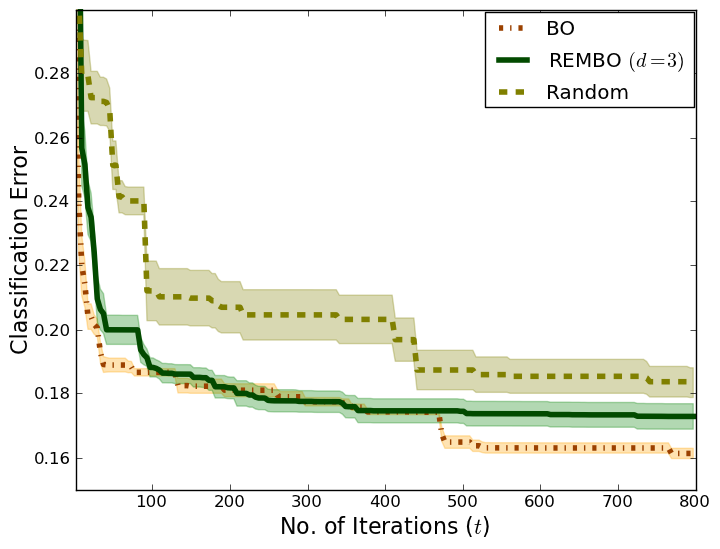}
  \caption{Performance of various methods for optimizing RF parameters for body part classification. For all methods, we show RF accuracy (mean $\pm$ 1/4 standard deviation across 10 runs) for all 2.2 million non background pixels in the 500-pose validation set, using the RF parameters identified by the method. The results on the test set were within 1\% of the results on the validation set.  Left: performance with a single run of each method; Right: performance with $k=4$ interleaved runs.}
  \label{fig:rf-results}
\end{figure}

Figure \ref{fig:rf-results} (left) shows the results that can be obtained by a single run of random search, BO, and REMBO. Remarkably, REMBO clearly outperformed random search, even based on as few as $d=3$ dimensions.\footnote{Due to the large computational expense of this experiment (in total over half a year of CPU time), we only performed conclusive experiments with $d=3$; preliminary runs of REMBO with $d=4$ performed somewhat worse than those with $d=3$ for a budget of 200 function evaluations, but were still improving at that point.}
However, since the extrinsic dimensionality was ``only'' a moderate $D=14$, standard Bayesian optimization performed well, and since it was not limited to a low-dimensional subspace it outperformed REMBO. 
Nevertheless, several REMBO runs actually performed very well, comparably with the best runs of BO. Consequently, when running $k=4$ interleaved runs of each method, REMBO performed almost as well as BO, matching its performance up to about 450 function evaluations (see Figure \ref{fig:rf-results}, right). 

We conclude that the parameter space of this RF classifier does not appear to have a clear low effective dimensionality;
since the extrinsic dimensionality is only moderate, this leads REMBO to perform somewhat worse than standard Bayesian optimization, but it is still possible to achieve reasonable performance based on as little as $d=3$ dimensions.

\section{Conclusion} \label{sec:conclusion}
We have demonstrated that it is possible to use random embeddings in Bayesian optimization to optimize functions 
of extremely high extrinsic dimensionality $D$ provided that they have low intrinsic dimensionality $d_e$. 
Moreover, our resulting REMBO algorithm is coordinate independent 
and it only requires a simple modification of the original Bayesian optimization algorithm; namely multiplication by a random matrix. 
We proved REMBO's independence of $D$ theoretically and empirically validated it by optimizing low-dimensional functions embedded in previously untenable extrinsic dimensionalities of up to $1$ billion. We also theoretically and empirically showed REMBO's rotational invariance.
Finally, we demonstrated that REMBO achieves state-of-the-art performance for optimizing the 47 discrete parameters of a popular mixed integer programming solver,
thereby providing further evidence for the observation (already put forward by Bergstra, Hutter and colleagues) that, for many problems of great practical interest, the number of important dimensions indeed appears to be much lower than their extrinsic dimensionality. 

We note that the central idea of our work -- using an otherwise unmodified optimization procedure in a randomly embedded space -- in principle could be applied to arbitrary optimization procedures. Evaluating the effciency of this technique for other procedures is an interesting topic for future work.


\section*{Acknowledgements}
We thank Christof Sch\"{o}tz for proofreading a draft of this article.

\appendix
\renewcommand{\baselinestretch}{1}



{


\bibliography{embedding}

\begin{thebibliography}{}

\bibitem[\protect\BCAY{Azimi, Fern,\ \BBA\ Fern}{Azimi
  et~al.}{2010}]{Azimi:2010}
Azimi, J., Fern, A., \BBA\ Fern, X. \BBOP2010\BBCP.
\newblock \BBOQ Batch {Bayesian} optimization via simulation matching\BBCQ\
\newblock In {\Bem Advances in Neural Information Processing Systems}, \BPGS\
  109--117.

\bibitem[\protect\BCAY{Azimi, Fern,\ \BBA\ Fern}{Azimi
  et~al.}{2011}]{Azimi:2011}
Azimi, J., Fern, A., \BBA\ Fern, X. \BBOP2011\BBCP.
\newblock \BBOQ Budgeted optimization with concurrent stochastic-duration
  experiments\BBCQ\
\newblock In {\Bem Advances in Neural Information Processing Systems}, \BPGS\
  1098--1106.

\bibitem[\protect\BCAY{Azimi, Jalali,\ \BBA\ Fern}{Azimi
  et~al.}{2012}]{Azimi:2012}
Azimi, J., Jalali, A., \BBA\ Fern, X. \BBOP2012\BBCP.
\newblock \BBOQ Hybrid batch {Bayesian} optimization\BBCQ\
\newblock In {\Bem International Conference on Machine Learning}.

\bibitem[\protect\BCAY{Bergstra, Bardenet, Bengio,\ \BBA\ K{\'e}gl}{Bergstra
  et~al.}{2011}]{Bergstra:2011}
Bergstra, J., Bardenet, R., Bengio, Y., \BBA\ K{\'e}gl, B. \BBOP2011\BBCP.
\newblock \BBOQ Algorithms for hyper-parameter optimization\BBCQ\
\newblock In {\Bem Advances in Neural Information Processing Systems}, \BPGS\
  2546--2554.

\bibitem[\protect\BCAY{Bergstra\ \BBA\ Bengio}{Bergstra\ \BBA\
  Bengio}{2012}]{Bergstra:2012}
Bergstra, J.\BBACOMMA\  \BBA\ Bengio, Y. \BBOP2012\BBCP.
\newblock \BBOQ Random search for hyper-parameter optimization\BBCQ\
\newblock {\Bem Journal of Machine Learning Research}, {\Bem 13}, 281--305.

\bibitem[\protect\BCAY{Bergstra, Yamins,\ \BBA\ Cox}{Bergstra
  et~al.}{2013}]{Bergstra:model_search}
Bergstra, J., Yamins, D., \BBA\ Cox, D.~D. \BBOP2013\BBCP.
\newblock \BBOQ Making a science of model search: Hyperparameter optimization
  in hundreds of dimensions for vision architectures\BBCQ\
\newblock In {\Bem International Conference on Machine Learning}, \BPGS\
  115--123.

\bibitem[\protect\BCAY{Berkelaar, Eikland,\ \BBA\ Notebaert}{Berkelaar
  et~al.}{2016}]{lpsolve}
Berkelaar, M., Eikland, K., \BBA\ Notebaert, P. \BBOP2016\BBCP.
\newblock {\Bem {lpsolve : Open source (Mixed-Integer) Linear Programming
  system}}.
\newblock \url{http://lpsolve.sourceforge.net/}.

\bibitem[\protect\BCAY{Brochu, Brochu,\ \BBA\ de~Freitas}{Brochu
  et~al.}{2010}]{Brochu:2010}
Brochu, E., Brochu, T., \BBA\ de~Freitas, N. \BBOP2010\BBCP.
\newblock \BBOQ A {Bayesian} interactive optimization approach to procedural
  animation design\BBCQ\
\newblock In {\Bem Proceedings of the 2010 ACM SIGGRAPH/Eurographics Symposium
  on Computer Animation}, \BPGS\ 103--112.

\bibitem[\protect\BCAY{Brochu, Cora,\ \BBA\ {de Freitas}}{Brochu
  et~al.}{2009}]{Brochu:2009}
Brochu, E., Cora, V.~M., \BBA\ {de Freitas}, N. \BBOP2009\BBCP.
\newblock \BBOQ A tutorial on {B}ayesian optimization of expensive cost
  functions, with application to active user modeling and hierarchical
  reinforcement learning\BBCQ\
\newblock \BTR\ UBC TR-2009-23 and arXiv:1012.2599v1, Dept. of Computer
  Science, University of British Columbia.

\bibitem[\protect\BCAY{Brochu, {de Freitas},\ \BBA\ Ghosh}{Brochu
  et~al.}{2007}]{Brochu:2007}
Brochu, E., {de Freitas}, N., \BBA\ Ghosh, A. \BBOP2007\BBCP.
\newblock \BBOQ Active preference learning with discrete choice data\BBCQ\
\newblock In {\Bem Advances in Neural Information Processing Systems}, \BPGS\
  409--416.

\bibitem[\protect\BCAY{Bubeck, Munos, Stoltz,\ \BBA\ Szepesvari}{Bubeck
  et~al.}{2011}]{Bubeck:2011}
Bubeck, S., Munos, R., Stoltz, G., \BBA\ Szepesvari, C. \BBOP2011\BBCP.
\newblock \BBOQ X-armed bandits\BBCQ\
\newblock {\Bem Journal of Machine Learning Research}, {\Bem 12}, 1655--1695.

\bibitem[\protect\BCAY{Bull}{Bull}{2011}]{Bull:2011}
Bull, A.~D. \BBOP2011\BBCP.
\newblock \BBOQ Convergence rates of efficient global optimization
  algorithms\BBCQ\
\newblock {\Bem Journal of Machine Learning Research}, {\Bem 12}, 2879--2904.

\bibitem[\protect\BCAY{Carpentier\ \BBA\ Munos}{Carpentier\ \BBA\
  Munos}{2012}]{Carpentier:2012}
Carpentier, A.\BBACOMMA\  \BBA\ Munos, R. \BBOP2012\BBCP.
\newblock \BBOQ Bandit theory meets compressed sensing for high dimensional
  stochastic linear bandit\BBCQ\
\newblock In {\Bem Artificial Intelligence and Statistics}, \BPGS\ 190--198.

\bibitem[\protect\BCAY{Chen, Castro,\ \BBA\ Krause}{Chen
  et~al.}{2012}]{Chen:2012}
Chen, B., Castro, R., \BBA\ Krause, A. \BBOP2012\BBCP.
\newblock \BBOQ Joint optimization and variable selection of high-dimensional
  {Gaussian} processes\BBCQ\
\newblock In {\Bem International Conference on Machine Learning}.

\bibitem[\protect\BCAY{{de Freitas}, Smola,\ \BBA\ Zoghi}{{de Freitas}
  et~al.}{2012}]{deFreitas:2012}
{de Freitas}, N., Smola, A., \BBA\ Zoghi, M. \BBOP2012\BBCP.
\newblock \BBOQ Exponential regret bounds for {Gaussian} process bandits with
  deterministic observations\BBCQ\
\newblock In {\Bem International Conference on Machine Learning}.

\bibitem[\protect\BCAY{Denil, Bazzani, Larochelle,\ \BBA\ {de Freitas}}{Denil
  et~al.}{2012}]{Denil:2012}
Denil, M., Bazzani, L., Larochelle, H., \BBA\ {de Freitas}, N. \BBOP2012\BBCP.
\newblock \BBOQ Learning where to attend with deep architectures for image
  tracking\BBCQ\
\newblock {\Bem Neural Computation}, {\Bem 24\/}(8), 2151--2184.

\bibitem[\protect\BCAY{Djolonga, Krause,\ \BBA\ Cevher}{Djolonga
  et~al.}{2013}]{Djolonga:2013}
Djolonga, J., Krause, A., \BBA\ Cevher, V. \BBOP2013\BBCP.
\newblock \BBOQ High dimensional {Gaussian} process bandits\BBCQ\
\newblock In {\Bem Advances in Neural Information Processing Systems}, \BPGS\
  1025--1033.

\bibitem[\protect\BCAY{Eggensperger, Feurer, Hutter, Bergstra, Snoek, Hoos,\
  \BBA\ Leyton-Brown}{Eggensperger et~al.}{2013}]{Eggensperger:2013}
Eggensperger, K., Feurer, M., Hutter, F., Bergstra, J., Snoek, J., Hoos, H.,
  \BBA\ Leyton-Brown, K. \BBOP2013\BBCP.
\newblock \BBOQ Towards an empirical foundation for assessing {B}ayesian
  optimization of hyperparameters\BBCQ\
\newblock In {\Bem NIPS Workshop on {B}ayesian Optimization in Theory and
  Practice}.

\bibitem[\protect\BCAY{Frazier, Powell,\ \BBA\ Dayanik}{Frazier
  et~al.}{2009}]{Frazier:2009}
Frazier, P., Powell, W., \BBA\ Dayanik, S. \BBOP2009\BBCP.
\newblock \BBOQ The knowledge-gradient policy for correlated normal
  beliefs\BBCQ\
\newblock {\Bem INFORMS journal on Computing}, {\Bem 21\/}(4), 599--613.

\bibitem[\protect\BCAY{Gomes, {van Hoeve},\ \BBA\ Sabharwal}{Gomes
  et~al.}{2008}]{ghs08:connection}
Gomes, C.~P., {van Hoeve}, W., \BBA\ Sabharwal, A. \BBOP2008\BBCP.
\newblock \BBOQ Connections in networks: A hybrid approach\BBCQ\
\newblock In {\Bem International Conference on Integration of Artificial
  Intelligence and Operations Research}, \lowercase{\BVOL}\ 5015, \BPGS\
  303--307.

\bibitem[\protect\BCAY{Gramacy, Lee,\ \BBA\ Macready}{Gramacy
  et~al.}{2004}]{Gramacy:2004}
Gramacy, R.~B., Lee, H. K.~H., \BBA\ Macready, W.~G. \BBOP2004\BBCP.
\newblock \BBOQ Parameter space exploration with {Gaussian} process trees\BBCQ\
\newblock In {\Bem International Conference on Machine Learning}, \BPGS\
  45--52.

\bibitem[\protect\BCAY{Gramacy\ \BBA\ Polson}{Gramacy\ \BBA\
  Polson}{2011}]{Gramacy:2011}
Gramacy, R.\BBACOMMA\  \BBA\ Polson, N. \BBOP2011\BBCP.
\newblock \BBOQ Particle learning of gaussian process models for sequential
  design and optimization\BBCQ\
\newblock {\Bem Journal of Computational and Graphical Statistics}, {\Bem
  20\/}(1), 102--118.

\bibitem[\protect\BCAY{Hamze, Wang,\ \BBA\ {de Freitas}}{Hamze
  et~al.}{2013}]{Hamze:2011}
Hamze, F., Wang, Z., \BBA\ {de Freitas}, N. \BBOP2013\BBCP.
\newblock \BBOQ Self-avoiding random dynamics on integer complex systems\BBCQ\
\newblock {\Bem ACM Transactions on Modelling and Computer Simulation}, {\Bem
  23\/}(1), 9:1--9:25.

\bibitem[\protect\BCAY{Hansen\ \BBA\ Ostermeier}{Hansen\ \BBA\
  Ostermeier}{2001}]{Hansen:2001:CDS:1108839.1108843}
Hansen, N.\BBACOMMA\  \BBA\ Ostermeier, A. \BBOP2001\BBCP.
\newblock \BBOQ Completely derandomized self-adaptation in evolution
  strategies\BBCQ\
\newblock {\Bem Evolutionary Computation}, {\Bem 9\/}(2), 159--195.

\bibitem[\protect\BCAY{Hennig\ \BBA\ Schuler}{Hennig\ \BBA\
  Schuler}{2012}]{Hennig:2012}
Hennig, P.\BBACOMMA\  \BBA\ Schuler, C. \BBOP2012\BBCP.
\newblock \BBOQ Entropy search for information-efficient global
  optimization\BBCQ\
\newblock {\Bem Journal of Machine Learning Research}, {\Bem 98888},
  1809--1837.

\bibitem[\protect\BCAY{Hoffman, Brochu,\ \BBA\ de~Freitas}{Hoffman
  et~al.}{2011}]{Hoffman:2011}
Hoffman, M., Brochu, E., \BBA\ de~Freitas, N. \BBOP2011\BBCP.
\newblock \BBOQ Portfolio allocation for {Bayesian} optimization\BBCQ\
\newblock In {\Bem Uncertainty in Artificial Intelligence}, \BPGS\ 327--336.

\bibitem[\protect\BCAY{Hoffman, Kueck, de~Freitas,\ \BBA\ Doucet}{Hoffman
  et~al.}{2009}]{Hoffman:2009}
Hoffman, M., Kueck, H., de~Freitas, N., \BBA\ Doucet, A. \BBOP2009\BBCP.
\newblock \BBOQ New inference strategies for solving {Markov} decision
  processes using reversible jump {MCMC}\BBCQ\
\newblock In {\Bem Uncertainty in Artificial Intelligence}, \BPGS\ 223--231.

\bibitem[\protect\BCAY{Hoffman, Shahriari,\ \BBA\ de~Freitas}{Hoffman
  et~al.}{2014}]{Hoffman:2014}
Hoffman, M., Shahriari, B., \BBA\ de~Freitas, N. \BBOP2014\BBCP.
\newblock \BBOQ On correlation and budget constraints in model-based bandit
  optimization with application to automatic machine learning\BBCQ\
\newblock In {\Bem Artificial Intelligence and Statistics}.

\bibitem[\protect\BCAY{Hoos}{Hoos}{2012}]{Hoos:2012:PO:2076450.2076469}
Hoos, H.~H. \BBOP2012\BBCP.
\newblock \BBOQ Programming by optimization\BBCQ\
\newblock {\Bem Communications of the ACM}, {\Bem 55\/}(2), 70--80.

\bibitem[\protect\BCAY{Hutter}{Hutter}{2009}]{Hutter:2009}
Hutter, F. \BBOP2009\BBCP.
\newblock {\Bem Automated Configuration of Algorithms for Solving Hard
  Computational Problems}.
\newblock Ph.D.\ thesis, University of British Columbia, Vancouver, Canada.

\bibitem[\protect\BCAY{Hutter, Hoos,\ \BBA\ Leyton-Brown}{Hutter
  et~al.}{2014}]{Hutter:2014}
Hutter, F., Hoos, H., \BBA\ Leyton-Brown, K. \BBOP2014\BBCP.
\newblock \BBOQ An efficient approach for assessing hyperparameter
  importance\BBCQ\
\newblock In {\Bem International Conference on Machine Learning}.

\bibitem[\protect\BCAY{Hutter, Hoos,\ \BBA\ Leyton-Brown}{Hutter
  et~al.}{2010}]{Hutter:2010}
Hutter, F., Hoos, H.~H., \BBA\ Leyton-Brown, K. \BBOP2010\BBCP.
\newblock \BBOQ Automated configuration of mixed integer programming
  solvers\BBCQ\
\newblock In {\Bem Conference on Integration of Artificial Intelligence and
  Operations Research}, \BPGS\ 186--202.

\bibitem[\protect\BCAY{Hutter, Hoos,\ \BBA\ Leyton-Brown}{Hutter
  et~al.}{2011}]{Hutter:2011}
Hutter, F., Hoos, H.~H., \BBA\ Leyton-Brown, K. \BBOP2011\BBCP.
\newblock \BBOQ Sequential model-based optimization for general algorithm
  configuration\BBCQ\
\newblock In {\Bem Learning and Intelligent Optimization}, \BPGS\ 507--523.

\bibitem[\protect\BCAY{Hutter, Hoos,\ \BBA\ Leyton-Brown}{Hutter
  et~al.}{2012}]{HutHooLey12-ParallelAC}
Hutter, F., Hoos, H.~H., \BBA\ Leyton-Brown, K. \BBOP2012\BBCP.
\newblock \BBOQ Parallel algorithm configuration\BBCQ\
\newblock In {\Bem Learning and Intelligent Optimization}, \BPGS\ 55--70.

\bibitem[\protect\BCAY{Hutter, Hoos,\ \BBA\ Leyton-Brown}{Hutter
  et~al.}{2013}]{HutHooLey13:BBOB}
Hutter, F., Hoos, H.~H., \BBA\ Leyton-Brown, K. \BBOP2013\BBCP.
\newblock \BBOQ An evaluation of sequential model-based optimization for
  expensive blackbox functions\BBCQ\
\newblock In {\Bem Proceedings of GECCO-13 Workshop on Blackbox Optimization
  Benchmarking (BBOB'13)}.

\bibitem[\protect\BCAY{Hutter, Hoos, Leyton-Brown,\ \BBA\ St\"{u}tzle}{Hutter
  et~al.}{2009}]{ParamILS-JAIR}
Hutter, F., Hoos, H.~H., Leyton-Brown, K., \BBA\ St\"{u}tzle, T.
  \BBOP2009\BBCP.
\newblock \BBOQ {ParamILS:} an automatic algorithm configuration
  framework\BBCQ\
\newblock {\Bem Journal of Artificial Intelligence Research}, {\Bem 36},
  267--306.

\bibitem[\protect\BCAY{Jones, Perttunen,\ \BBA\ Stuckman}{Jones
  et~al.}{1993}]{Jones:1993}
Jones, D.~R., Perttunen, C.~D., \BBA\ Stuckman, B.~E. \BBOP1993\BBCP.
\newblock \BBOQ Lipschitzian optimization without the {L}ipschitz
  constant\BBCQ\
\newblock {\Bem J. of Optimization Theory and Applications}, {\Bem 79\/}(1),
  157--181.

\bibitem[\protect\BCAY{Jones}{Jones}{2001}]{Jones:2001}
Jones, D. \BBOP2001\BBCP.
\newblock \BBOQ A taxonomy of global optimization methods based on response
  surfaces\BBCQ\
\newblock {\Bem Journal of Global Optimization}, {\Bem 21\/}(4), 345--383.

\bibitem[\protect\BCAY{Jones, Schonlau,\ \BBA\ Welch}{Jones
  et~al.}{1998}]{Jones:1998}
Jones, D., Schonlau, M., \BBA\ Welch, W. \BBOP1998\BBCP.
\newblock \BBOQ Efficient global optimization of expensive black-box
  functions\BBCQ\
\newblock {\Bem Journal of Global optimization}, {\Bem 13\/}(4), 455--492.

\bibitem[\protect\BCAY{Kueck, de~Freitas,\ \BBA\ Doucet}{Kueck
  et~al.}{2006}]{Kueck:2006}
Kueck, H., de~Freitas, N., \BBA\ Doucet, A. \BBOP2006\BBCP.
\newblock \BBOQ {SMC} samplers for {Bayesian} optimal nonlinear design\BBCQ\
\newblock In {\Bem IEEE Nonlinear Statistical Signal Processing Workshop},
  \BPGS\ 99--102.

\bibitem[\protect\BCAY{Kueck, Hoffman, Doucet,\ \BBA\ de~Freitas}{Kueck
  et~al.}{2009}]{Kueck:2009}
Kueck, H., Hoffman, M., Doucet, A., \BBA\ de~Freitas, N. \BBOP2009\BBCP.
\newblock \BBOQ Inference and learning for active sensing, experimental design
  and control\BBCQ\
\newblock In {\Bem Pattern Recognition and Image Analysis}, \lowercase{\BVOL}\
  5524, \BPGS\ 1--10.

\bibitem[\protect\BCAY{Lizotte, Greiner,\ \BBA\ Schuurmans}{Lizotte
  et~al.}{2011}]{Lizotte:2011}
Lizotte, D., Greiner, R., \BBA\ Schuurmans, D. \BBOP2011\BBCP.
\newblock \BBOQ An experimental methodology for response surface optimization
  methods\BBCQ\
\newblock {\Bem Journal of Global Optimization}, {\Bem 53\/}(4), 1--38.

\bibitem[\protect\BCAY{Lizotte, Wang, Bowling,\ \BBA\ Schuurmans}{Lizotte
  et~al.}{2007}]{Lizotte:2008_IJCAI}
Lizotte, D., Wang, T., Bowling, M., \BBA\ Schuurmans, D. \BBOP2007\BBCP.
\newblock \BBOQ Automatic gait optimization with {G}aussian process
  regression\BBCQ\
\newblock In {\Bem International Joint Conference on Artificial Intelligence},
  \BPGS\ 944--949.

\bibitem[\protect\BCAY{Mahendran, Wang, Hamze,\ \BBA\ {de Freitas}}{Mahendran
  et~al.}{2012}]{Mahendran:2012}
Mahendran, N., Wang, Z., Hamze, F., \BBA\ {de Freitas}, N. \BBOP2012\BBCP.
\newblock \BBOQ Adaptive {MCMC} with {Bayesian} optimization\BBCQ\
\newblock {\Bem Journal of Machine Learning Research - Proceedings Track},
  {\Bem 22}, 751--760.

\bibitem[\protect\BCAY{Marchant\ \BBA\ Ramos}{Marchant\ \BBA\
  Ramos}{2012}]{Marchant:2012}
Marchant, R.\BBACOMMA\  \BBA\ Ramos, F. \BBOP2012\BBCP.
\newblock \BBOQ {Bayesian} optimisation for intelligent environmental
  monitoring\BBCQ\
\newblock In {\Bem NIPS workshop on {Bayesian} Optimization and Decision
  Making}.

\bibitem[\protect\BCAY{{Martinez-Cantin}, {de Freitas}, Doucet,\ \BBA\
  Castellanos}{{Martinez-Cantin} et~al.}{2007}]{martinez-cantin:2007}
{Martinez-Cantin}, R., {de Freitas}, N., Doucet, A., \BBA\ Castellanos, J.~A.
  \BBOP2007\BBCP.
\newblock \BBOQ Active policy learning for robot planning and exploration under
  uncertainty\BBCQ\
\newblock In {\Bem Robotics, Science and Systems}.

\bibitem[\protect\BCAY{Mo{\v c}kus}{Mo{\v c}kus}{1982}]{Mockus:1982}
Mo{\v c}kus, J. \BBOP1982\BBCP.
\newblock \BBOQ The {B}ayesian approach to global optimization\BBCQ\
\newblock In {\Bem Systems Modeling and Optimization}, \lowercase{\BVOL}~38,
  \BPGS\ 473--481. Springer.

\bibitem[\protect\BCAY{Mo{\v c}kus}{Mo{\v c}kus}{1994}]{Mockus:1994}
Mo{\v c}kus, J. \BBOP1994\BBCP.
\newblock \BBOQ Application of {Bayesian} approach to numerical methods of
  global and stochastic optimization\BBCQ\
\newblock {\Bem J. of Global Optimization}, {\Bem 4\/}(4), 347--365.

\bibitem[\protect\BCAY{Mo{\v c}kus, Mo{\v c}kus,\ \BBA\ Mo{\v c}kus}{Mo{\v
  c}kus et~al.}{1999}]{Mockus:1999}
Mo{\v c}kus, J., Mo{\v c}kus, A., \BBA\ Mo{\v c}kus, L. \BBOP1999\BBCP.
\newblock {\Bem Bayesian approach for randomization of heuristic algorithms of
  discrete programming}.
\newblock American Math. Society.

\bibitem[\protect\BCAY{Osborne, Garnett,\ \BBA\ Roberts}{Osborne
  et~al.}{2009}]{Osborne:2009}
Osborne, M.~A., Garnett, R., \BBA\ Roberts, S.~J. \BBOP2009\BBCP.
\newblock \BBOQ Gaussian processes for global optimisation\BBCQ\
\newblock In {\Bem Learning and Intelligent Optimization}, \BPGS\ 1--15.

\bibitem[\protect\BCAY{Rasmussen}{Rasmussen}{2003}]{Rasmussen:2003}
Rasmussen, C.~E. \BBOP2003\BBCP.
\newblock \BBOQ Gaussian processes to speed up hybrid {Monte Carlo} for
  expensive {Bayesian} integrals\BBCQ\
\newblock In {\Bem Bayesian Statistics 7}.

\bibitem[\protect\BCAY{Rasmussen\ \BBA\ Williams}{Rasmussen\ \BBA\
  Williams}{2006}]{Rasmussen:2006}
Rasmussen, C.~E.\BBACOMMA\  \BBA\ Williams, C. K.~I. \BBOP2006\BBCP.
\newblock {\Bem Gaussian Processes for Machine Learning}.
\newblock The MIT Press.

\bibitem[\protect\BCAY{Rudelson\ \BBA\ Vershynin}{Rudelson\ \BBA\
  Vershynin}{2010}]{Rudelson:2010}
Rudelson, M.\BBACOMMA\  \BBA\ Vershynin, R. \BBOP2010\BBCP.
\newblock \BBOQ Non-asymptotic theory of random matrices: {Extreme} singular
  values\BBCQ\
\newblock In {\Bem International Congress of Mathematicians}, \BPGS\
  1576--1599.

\bibitem[\protect\BCAY{Sankar, Spielman,\ \BBA\ Teng}{Sankar
  et~al.}{2003}]{Sankar:2003}
Sankar, A., Spielman, D., \BBA\ Teng, S. \BBOP2003\BBCP.
\newblock \BBOQ Smoothed analysis of the condition numbers and growth factors
  of matrices\BBCQ\
\newblock \BTR\ Arxiv preprint cs/0310022, MIT.

\bibitem[\protect\BCAY{Shotton, Fitzgibbon, Cook, Sharp, Finocchio, Moore,
  Kipman,\ \BBA\ Blake}{Shotton et~al.}{2011}]{Shotton:2011}
Shotton, J., Fitzgibbon, A., Cook, M., Sharp, T., Finocchio, M., Moore, R.,
  Kipman, A., \BBA\ Blake, A. \BBOP2011\BBCP.
\newblock \BBOQ Real-time human pose recognition in parts from single depth
  images\BBCQ\
\newblock In {\Bem IEEE Computer Vision and Pattern Recognition}, \BPGS\
  1297--1304.

\bibitem[\protect\BCAY{Snoek, Larochelle,\ \BBA\ Adams}{Snoek
  et~al.}{2012}]{Snoek:2012}
Snoek, J., Larochelle, H., \BBA\ Adams, R.~P. \BBOP2012\BBCP.
\newblock \BBOQ Practical {Bayesian} optimization of machine learning
  algorithms\BBCQ\
\newblock In {\Bem Advances in Neural Information Processing Systems}, \BPGS\
  2960--2968.

\bibitem[\protect\BCAY{Srinivas, Krause, Kakade,\ \BBA\ Seeger}{Srinivas
  et~al.}{2010}]{Srinivas:2010}
Srinivas, N., Krause, A., Kakade, S.~M., \BBA\ Seeger, M. \BBOP2010\BBCP.
\newblock \BBOQ Gaussian process optimization in the bandit setting: No regret
  and experimental design\BBCQ\
\newblock In {\Bem International Conference on Machine Learning}, \BPGS\
  1015--1022.

\bibitem[\protect\BCAY{Steinwart\ \BBA\ Christmann}{Steinwart\ \BBA\
  Christmann}{2008}]{Steinwart:2008}
Steinwart, I.\BBACOMMA\  \BBA\ Christmann, A. \BBOP2008\BBCP.
\newblock {\Bem Support Vector Machines}.
\newblock Springer.

\bibitem[\protect\BCAY{Swersky, Snoek,\ \BBA\ Adams}{Swersky
  et~al.}{2013}]{Swersky:2013}
Swersky, K., Snoek, J., \BBA\ Adams, R.~P. \BBOP2013\BBCP.
\newblock \BBOQ Multi-task {Bayesian} optimization\BBCQ\
\newblock In {\Bem Advances in Neural Information Processing Systems}, \BPGS\
  2004--2012.

\bibitem[\protect\BCAY{Thompson}{Thompson}{1933}]{Thompson:1933}
Thompson, W.~R. \BBOP1933\BBCP.
\newblock \BBOQ On the likelihood that one unknown probability exceeds another
  in view of the evidence of two samples\BBCQ\
\newblock {\Bem Biometrika}, {\Bem 25\/}(3/4), 285--294.

\bibitem[\protect\BCAY{Thornton, Hutter, Hoos,\ \BBA\ Leyton-Brown}{Thornton
  et~al.}{2013}]{Thornton:2013}
Thornton, C., Hutter, F., Hoos, H.~H., \BBA\ Leyton-Brown, K. \BBOP2013\BBCP.
\newblock \BBOQ Auto-{WEKA}: Combined selection and hyperparameter optimization
  of classification algorithms\BBCQ\
\newblock In {\Bem ACM SIGKDD Conference on Knowledge Discovery and Data
  Mining}, \BPGS\ 847--855.

\bibitem[\protect\BCAY{Vallati, Fawcett, Gerevini, Hoos,\ \BBA\ Saetti}{Vallati
  et~al.}{2011}]{ValEtAl11}
Vallati, M., Fawcett, C., Gerevini, A.~E., Hoos, H.~H., \BBA\ Saetti, A.
  \BBOP2011\BBCP.
\newblock \BBOQ Generating fast domain-optimized planners by automatically
  configuring a generic parameterised planner\BBCQ\
\newblock In {\Bem ICAPS Planning and Learning Workshop}.

\bibitem[\protect\BCAY{Vazquez\ \BBA\ Bect}{Vazquez\ \BBA\
  Bect}{2010}]{Vazquez:2011}
Vazquez, E.\BBACOMMA\  \BBA\ Bect, J. \BBOP2010\BBCP.
\newblock \BBOQ Convergence properties of the expected improvement algorithm
  with fixed mean and covariance functions\BBCQ\
\newblock {\Bem Journal of Statistical Planning and Inference}, {\Bem 140},
  3088--3095.

\bibitem[\protect\BCAY{Wang\ \BBA\ de~Freitas}{Wang\ \BBA\
  de~Freitas}{2011}]{Wang:2011}
Wang, Z.\BBACOMMA\  \BBA\ de~Freitas, N. \BBOP2011\BBCP.
\newblock \BBOQ Predictive adaptation of hybrid {Monte Carlo} with {Bayesian}
  parametric bandits\BBCQ\
\newblock In {\Bem NIPS Deep Learning and Unsupervised Feature Learning
  Workshop}.

\bibitem[\protect\BCAY{Wang\ \BBA\ de~Freitas}{Wang\ \BBA\
  de~Freitas}{2014}]{Wang:2014}
Wang, Z.\BBACOMMA\  \BBA\ de~Freitas, N. \BBOP2014\BBCP.
\newblock \BBOQ Bayesian multiscale optimistic optimization\BBCQ\
\newblock In {\Bem Artificial Intelligence and Statistics}.

\bibitem[\protect\BCAY{Wang, Zoghi, Hutter, Matheson,\ \BBA\ de~Freitas}{Wang
  et~al.}{2013}]{WangEtAl13}
Wang, Z., Zoghi, M., Hutter, F., Matheson, D., \BBA\ de~Freitas, N.
  \BBOP2013\BBCP.
\newblock \BBOQ Bayesian optimization in high dimensions via random
  embeddings\BBCQ\
\newblock In {\Bem International Joint Conference on Artificial Intelligence},
  \BPGS\ 1778--1784.

\end{thebibliography}
\bibliographystyle{theapa}
}




\section{Proof of Theorem~\ref{prop:1}}

\begin{proof}
Since $f$ has effective dimensionality $d_e$, there exists an effective subspace ${\cal T} \subset \mathbb{R}^D$, such that rank$({\cal T}) = d_e$. Furthermore, any $\vx \in \mathbb{R}^D$ decomposes as
 $\vx = \vx_{\top} + \vx_{\bot}$, where $\vx_{\top} \in {\cal T}$ and $\vx_{\bot} \in {\cal T}^{\bot}$. Hence, $f(\vx) = f(\vx_{\top} + \vx_{\bot}) = f(\vx_{\top}).$ Therefore, without loss of generality, it will suffice to show that for all $\vx_{\top} \in {\cal T}$, there exists a $\vy\in \mathbb{R}^d$ such that $f(\vx_{\top}) = f(\vA\vy)$.

Let $\vPhi \in \mathbb{R}^{D \times d_e}$ be a matrix, whose columns form an orthonormal basis for ${\cal T}$. Hence, for each $\vx_{\top} \in {\cal T}$, there exists a $\vc \in \mathbb{R}^{d_e}$ such that $\vx_{\top} = \vPhi \vc$. Let us for now assume that $\vPhi^{T}\vA$ has rank $d_e$. If $\vPhi^{T}\vA$ has rank $d_e$, there exists a $\vy$ such that $(\vPhi^{T}\vA)\vy=\vc$. The orthogonal projection of $\vA\vy$ onto ${\cal T}$ is given by 
\[ \vPhi \vPhi^T \vA \vy = \vPhi \vc = \vx_{\top}. \]
Thus $\vA\vy = \vx_{\top} + \vx'$ for some $\vx' \in {\cal T}^{\bot}$ since $\vx_{\top}$ is the projection $\vA\vy$ onto ${\cal T}$.
Consequently, $f(\vA\vy) = f(\vx_{\top} + \vx') = f(\vx_{\top})$. 

It remains to show that, with probability one, the matrix $\vPhi^T \vA$ has rank $d_e$.
Let $\vA_{e} \in \mathbb{R}^{D\times d_e}$ be a submatrix of $\vA$ consisting of any $d_e$ columns of $\vA$, which are \emph{i.i.d.} samples distributed according to $\mathcal{N}(\mathbf{0}, \vI)$. Then, $\vPhi^T \va_{i}$ are \emph{i.i.d.} samples from $\mathcal{N}(\mathbf{0}, \vPhi^T \vPhi) = \mathcal{N}(\mathbf{0}_{d_e}, \vI_{d_e\times d_e})$, and so we have $\vPhi^T\vA_e$, when considered as an element of $\mathbb{R}^{d_e^2}$, is a sample from $\mathcal{N}(\mathbf{0}_{d_e^2}, \vI_{d_e^2\times d_e^2})$. On the other hand, the set of singular matrices in $\mathbb{R}^{d_e^2}$ has Lebesgue measure zero, since it is the zero set of a polynomial (i.e. the determinant function) and polynomial functions are Lebesgue measurable. Moreover, the Normal distribution is absolutely continuous with respect to the Lebesgue measure, so our matrix $\vPhi^T \vA_e$ is almost surely non-singular, which means that it has rank $d_e$ and so the same is true of $\vPhi^T \vA$, whose columns contain the columns of $\vPhi^T \vA_e$.
\end{proof}

\section{Proof of Theorem~\ref{prop:2}}

\begin{proof}
Since $\mathcal{X}$ is a box constraint, by projecting $\vx^\star$ to $\cal T$ we get $\vx^\star_\top \in \mathcal{T} \cap \mathcal{X}$. Also, since $\vx^\star =  \vx^\star_\top + \vx_\bot$ for some $\vx_\bot \in \cal T^{\bot}$, we have $f(\vx^\star) = f(\vx^\star_\top)$. Hence, $\vx^\star_\top$ is an optimizer.
By using the same argument as appeared in Proposition 1, it is easy to see that with probability $1$ $\forall \vx \in \cal T$ $\exists \vy \in \mathbb{R}^d$ such that $\vA\vy = \vx + \vx_\bot$ where $\vx_\bot \in \cal T^{\bot}$. Let $\vPhi$ be the matrix whose columns form a standard basis for $\cal T$. Without loss of generality, we can assume that 
\[ \vPhi = \begin{bmatrix}\vI_{d_e} \\
\mathbf{0} \end{bmatrix} \]
Then, as shown in Proposition~\ref{prop:1}, there exists a $\vy^\star \in \mathbb{R}^d$ such that $\vPhi \vPhi^{T}\vA \vy^\star = \vx^\star_\top$. Note that for each column of $\vA$, we have
\[ \vPhi \vPhi^{T}\va_i \sim \mathcal{N}\left(\mathbf{0}, \begin{bmatrix}
\vI_{d_e} & {\bf 0}\\
{\bf 0} & {\bf 0}\end{bmatrix}\right). \]
Therefore $\vPhi \vPhi^{T}\vA\vy^\star = \vx^\star_\top$ is equivalent to $\vB\vy^\star = \bar{\vx}^\star_\top$ where $\vB\in \mathbb{R}^{d_e \times d_e}$ is a random matrix with independent standard Gaussian entries and $\bar{\vx}^\star_\top$ is the vector that contains the first $d_e$ entries of $\vx^\star_\top$ (the rest are $0$'s). By Theorem 3.4 of \cite{Sankar:2003}, we have
\[ \mathbb{P}\left[\|\vB^{-1}\|_2 \geq \frac{\sqrt{d_e}}{\epsilon}\right]  \leq \epsilon. \]

Thus, with probability at least $1-\epsilon$, $\|\vy^\star\| \leq \|\vB^{-1}\|_2  \|\bar{\vx}^\star_\top\|_2 = \|\vB^{-1}\|_2  \|\vx^\star_\top\|_2 \leq \frac{\sqrt{d_e}}{\epsilon} \|\vx^\star_\top\|_2$. 
\end{proof}

\section{Regret Bounds}


In this section, we provide regret results for REMBO in the special case that (1)~the embedded subspace has the same dimension as the effective dimension and (2)~the embedded subspace contains a maximum of the function $f$ inside the box $\mathcal{X}$. More specifically, here we will analyze a simplified version of the algorithm that performs Bayesian optimization only inside the box $\mathcal{X}$ rather than considering its extension beyond $\mathcal{X}$ and projecting onto the boundary of $\mathcal{X}$ as done in our actual implementation. 

We acknowledge that this mismatch between the theoretical results and our actual algorithm is rather unsatisfactory. However, some of the obstacles that stand in the way of a complete analysis of the algorithm are currently insurmountable, since they would require the development of new tools that are far beyond the scope of this paper. We point these out at the end of this section and hope that our partial result will motivate the development of such tools, which might not otherwise receive any attention from the community.

We begin our mathematical treatment with the definitions of \emph{simple regret} and the \emph{skew squared exponential (SSE)} kernel.
\begin{mydefinition}\label{def:regret} Given a function $f: {\mathcal{X}} \to \mathbb{R}$ and a sequence of points $\{\vx_t\}_{t=1}^\infty \subseteq \mathcal{X}$, the \emph{simple regret} {\bf with respect to the set} $\mathcal{X}$ at time $T$ is defined to be $r_f(T) =  \sup_{\mathcal{X}}f - \displaystyle\max_{t=1}^T f(\vx_t)$.
\end{mydefinition}
\begin{mydefinition}\label{def:ssekernel}
Given a symmetric, positive-definite matrix $\mathbf{\vDelta}$ 
and $K_{SE}$, we define the corresponding \emph{skew squared exponential} kernel as
$$ k_{\vDelta}(\vy^{(1)},\vy^{(2)}) = 
K_{SE}\left(\mathbf{\vDelta}^{-1/2} (\vy^{(1)}-\vy^{(2)})\right). $$
\end{mydefinition}
Given $\vDelta$,
and $\mathcal{X} \subseteq \mathbb R^d$, we denote the Reproducing Kernel Hilbert Spaces (RKHSs) corresponding to $k_{\vDelta}$ by $\calH_{k_{\vDelta}}$ \cite[Definition 4.18]{Steinwart:2008}.
Note that for the setting $\vDelta = \ell^2 I$ we recover the squared exponential kernel $k_{\ell}^d$ introduced earlier.
To simplify notation, we introduce yet another definition.
\begin{mydefinition}\label{def:canrep}
Consider the function $f: \calS \rightarrow \mathbb{R}$, where $\calS \subset \mathbb{R}^D$
is a $d$-dimensional subspace of $\mathbb{R}^D$. 
Let $\vPhi \in \mathbb{R}^{D \times d}$ be a matrix 
whose columns form an orthonormal basis for ${\cal S}$.
We define the \emph{canonical representation} $\overline{f}: \mathbb{R}^d \rightarrow \mathbb{R}$ 
of $f$ as $\overline{f}(\vx) = f(\vPhi \vx)$.
\end{mydefinition}

Our main result below shows that the simple regret of a variation of REMBO vanishes with rate $\mathcal O(t^{-\frac{1}{d}})$ with high probability. This REMBO variant uses a fixed kernel parameter length scale $\ell$ and, more importantly, restricts its search to its embedding inside the box $\mathcal{X}$.
 
We only make the assumption that the cost function restricted to $\mathcal T$ is governed by a skew squared exponential kernel, a much weaker assumption than the standard assumption that the cost function is governed by an axis aligned kernel in $D$ dimensions \cite<see, e.g.,>{Bull:2011}. Despite the fact that the cost function restricted to $\mathcal T$ is governed by a skew squared exponential kernel, the result shows that we can control the regret using the low-dimensional squared exponential kernel.

\begin{theorem}\label{thm:regret}
Let $\mathcal{X} \subset \mathbb{R}^D$ be a compact subset with non-empty interior that is convex and contains the origin and $f:\mathbb{R}^D \to \mathbb{R}$, a function with effective dimension $d$. 
Suppose that the canonical representation of the restriction of $f$ to its effective subspace $\mathcal T$, denoted $\overline{f|_\mathcal T}$, is an element of the RKHS $\calH_{k_{\vDelta}}(\mathbb R^d)$ with $\vDelta$ symmetric and positive definite and also satisfying $0 < r^2 < \lambda_{\min}(\vDelta) \leq \lambda_{\max}(\vDelta) < R^2$ for constants $r$ and $R$, where $\lambda_{\min}(\vDelta)$ and $\lambda_{\max}(\vDelta)$ are the extreme eigenvalues of $\vDelta$.

Let $\vA$ be a $D \times d$ matrix, whose elements are drawn from the normal distribution $\mathcal N\left(0,1\right)$. 
Then, given any $\epsilon > 0$, we can choose a length-scale $\ell = \ell(\epsilon)$ such that running REMBO with kernel $k^d_{\ell}$ on the restriction of $f$ to the image of $\vA$ inside $\mathcal{X}$ has simple regret {\bf with respect to the set} $\textup{Im } \vA \cap \mathcal{X}$ in $\mathcal O(t^{-\frac{1}{d}})$ with probability $1-\epsilon$.
\end{theorem}

This theorem does not follow directly from the results of \citeA{Bull:2011}, since the kernel is not aligned with the axes, both in the high-dimensional space and the lower dimensional embedding. 

Please refer to Appendix \ref{sec:regret_proof} for the proof of this theorem. 
The general idea of the proof is as follows. If we have a squared exponential kernel $k_\ell$, with a smaller length scale than a given kernel $k_{\vDelta}$, then an element $f$ of the RKHS of $k_{\vDelta}$
is also an element of the RKHS of $k_\ell$ (see Lemma~\ref{lem:bull4} in the Appendix for more details). 
So, when running expected improvement, one can safely use $k_\ell$ instead of $k_{\vDelta}$ as the kernel and still obtain a regret bound. Most of the proof is dedicated to finding a length scale $\ell$ that fits ``underneath'' our kernel, so we can replace our kernel with $k_\ell$, to which we can apply the results of \citeA{Bull:2011}.

Note that in the above theorem we make the assumption that the embedded dimension and the effective dimension are equal to each other. Given bounds such as Proposition 1 of \citeA{deFreitas:2012}, we strongly believe that a similar result holds when the embedded dimension is higher than the effective dimension; however, the analysis of that setting remains elusive due to the fact that none of the methods available in the literature on regret bounds for Bayesian optimization algorithms can handle kernels that have flat dimensions (i.e. when $\vDelta$ is not positive-definite), and adapting them to such a case requires tools from statistics that have yet to be developed. Given that, theoretical bounds for the case with $d_e < d$ are outside the scope of this work.

Moreover, note that this theorem provides a sublinear regret result for REMBO with respect to the whole set $\mathcal{X}$ only in the situation that $\textup{Im }\vA$ intersects the maximum locus of the function $f$ \emph{inside} the set $\mathcal{X}$. Note that Theorem \ref{prop:2} provides a lower bound on the probability of this happening in the special case that the effective subspace is axis-aligned.
%
Proving regret bounds for situations in which the image of $\vA$ only contains a maximum outside of $\mathcal{X}$ would require dealing with the RKHS of non-stationary kernels, since the projection operator $p_{\mathcal{X}}$ can have non-constant Jacobian. This is related to the situation with treed GPs~\cite{Gramacy:2004}, with the additional, immensely complicating ingredient that is the continuity assumption imposed along the boundaries of the various partitions of the space (since $p_{\mathcal{X}}$ is continuous). Similar to the theory of Partial Differential Equations, where boundary conditions are the hardest part of the problem, we anticipate this modification to be a non-trivial, albeit very interesting, undertaking, and pose it as an open problem to the community.


\begin{remark}
The above theorem would also hold for a class of stationary kernels which includes 
the popular Mat\'ern kernel. 
For conciseness of presentation, we do not include this result, but refer 
the curious reader to~\citeA{Bull:2011} for more details.
\end{remark}

\section{Proof of Theorem~\ref{thm:regret}}\label{sec:regret_proof}
Before embarking on the proof of Theorem \ref{thm:regret}, 
we introduce some definitions and state a few preliminary results,
which we quote from \citeA{Bull:2011} to facilitate the reading of this exposition.

We denote the Fourier transform of any function 
$\phi(\vx)$ as $\widehat{\phi}(\vxi) = \int_{\mathbb{R}^d}{e^{-2\pi i \vx^T \vxi}\phi(\vx)d\vx}$.
In this section we consider kernels of the form
$$
k_{\vDelta}(\vx^{(1)}, \vx^{(2)}) =  
K_{\vDelta}(\vx^{(1)}-\vx^{(2)}) = 
K(\vDelta^{-\frac{1}{2}} (\vx^{(1)}-\vx^{(2)}))
$$
where $\vDelta$ is a positive definite matrix and $K$ has Fourier transform 
$\widehat{K}$
such that $\widehat{K}$ is isotropic and radially non-increasing.
Notice that both the squared exponential kernel and the skew squared exponential
kernel introduced in Definition \ref{def:ssekernel} of the main text are represented in the form above.
The popular kernels from the Mat\'ern class can also be represented in this form. 
In general, the results in this section would follow for any kernel
that satisfies the four assumptions detailed by~\citeA{Bull:2011}.
\begin{lemma}[Lemma 1 of \citeA{Bull:2011}]
 \label{lem:bull1}
 $\calH(\mathbb{R}^d)$ is the space of real continuous functions 
 $f \in L^2(\mathbb{R}^d)$ whose norm
 $$
 \|f\|^2_{\calH(\mathbb{R}^d)} := 
 \int{\frac{ | \widehat{f}(\vxi) |^2 }{\widehat{K}(\vxi)}d\vxi}
 $$
 is finite, taking $0/0=0$.
\end{lemma}

\begin{lemma}[Lemma 2 of \citeA{Bull:2011}]
 \label{lem:bull2}
 Given a set $\calS \subseteq \mathbb{R}^d$,
 $\calH(\calS)$ is the space of functions $f = g|\calS$ ($f$ is $g$ restricted to $\calS$) 
 for some $g \in \calH(\mathbb{R}^d)$, with norm
 $$
 \|f\|_{\calH(\calS)} := \inf_{g|\calS = f} \|g\|_{\calH(\mathbb{R}^d)},
 $$
 and there is a unique $g$ minimizing this expression.
\end{lemma}

\begin{lemma}[Lemma 4 of \citeA{Bull:2011}, extended to our setting]
 \label{lem:bull4}
 Let $\calS \subseteq \mathbb{R}^d$
 and $\vDelta$ and $\vDelta'$ be two symmetric positive definite matrices.
 Let $\lambda_{\max}( \vDelta') $, $ \lambda_{\min}(\vDelta)$ 
 be the largest and the smallest eigenvalues of $\vDelta'$, $\vDelta$ respectively
 such that $\lambda_{\max}( \vDelta') \leq \lambda_{\min}(\vDelta)$.
 Then $f \in \calH_{k_{\vDelta}}(\calS)$  implies 
 $f \in \calH_{k_{\vDelta'}}(\calS)$ 
 and also
 $$\|f\|_{\calH_{k_{\vDelta'}}(\calS)} 
 \leq \left(\frac{| \vDelta | }{ | \vDelta'|}\right)^{\frac{1}{2}} 
 \|f\|_{\calH_{k_{\vDelta}}(\calS)}$$
 where $|\vDelta|$ is the determinant of $\vDelta$.
\end{lemma}
\begin{proof}
 Since $\vDelta$ and $\vDelta'$ are positive definite we can write 
 $\vDelta = Q^T \Sigma Q$ and $\vDelta' = Q'^T \Sigma' Q'$ 
 where $Q$ and $Q'$ are orthonormal.
 Let $C = \left(\frac{| \vDelta' | }{ | \vDelta|}\right)^{\frac{1}{2}}$. 
 Since $\lambda_{\max}( \vDelta') \leq \lambda_{\min}(\vDelta)$, we know that 
 $$
 \|\Sigma'^{\frac{1}{2}} Q' \vxi \|^2
 =    \vxi^T \vDelta' \vxi 
 \leq \vxi^T \vDelta \vxi
 =    \|\Sigma^{\frac{1}{2}} Q \vxi \|^2.
 $$
 As $\widehat{K}$ is isotropic and radially non-increasing,
 we have that 
 $$
 \widehat{K_{\vDelta'}}(\vxi)
 =    |\vDelta'|^{\frac{1}{2}} \widehat{K}( \Sigma'^{\frac{1}{2}}Q'\vxi) 
 \geq |\vDelta'|^{\frac{1}{2}} \widehat{K}( \Sigma^{\frac{1}{2}}Q\vxi)
 =    |\vDelta'|^{\frac{1}{2}} |\vDelta|^{-\frac{1}{2}} \widehat{K_{\vDelta}}(\vxi)
 =    C \widehat{K_{\vDelta}}(\vxi)
 $$
 where the first and the second last equality follows from the following property of Fourier transforms: $\widehat{h}(\vxi) = \frac{1}{|M|}\widehat{f}(M^{-T}\vxi)$ 
 if $h(\vx) = f(M\vx)$ given a non-singular matrix $M$.
 Given $f \in \calH_{k_{\vDelta}}(\calS)$, let $g \in \calH_{k_{\vDelta}}(\mathbb{R}^d)$
 be its minimum norm extension, as in Lemma~\ref{lem:bull2}. 
 By the definition of RKHS norm in Lemma~\ref{lem:bull1}, 
 $$
 \|f\|^2_{\calH_{k_{\vDelta'}}(\calS)} 
 = \|g\|^2_{\calH_{k_{\vDelta'}}(\mathbb{R}^d)}
 =    \int{\frac{ | \widehat{g} |^2 }{\widehat{K_{\vDelta'}}}}
 \leq \int{\frac{ | \widehat{g} |^2 }{C \widehat{K_{\vDelta}}}}
 =    C^{-1}\|f\|^2_{\calH_{k_{\vDelta}}(\calS)}.
 $$
 Since $C$ is finite, by Lemma~\ref{lem:bull1} we have that
 $f \in \calH_{k_{\vDelta}}(\calS)$  implies $f \in \calH_{k_{\vDelta'}}(\calS)$.
\end{proof}

\begin{figure}[t!]
\centering
  \includegraphics[width=0.8\textwidth]{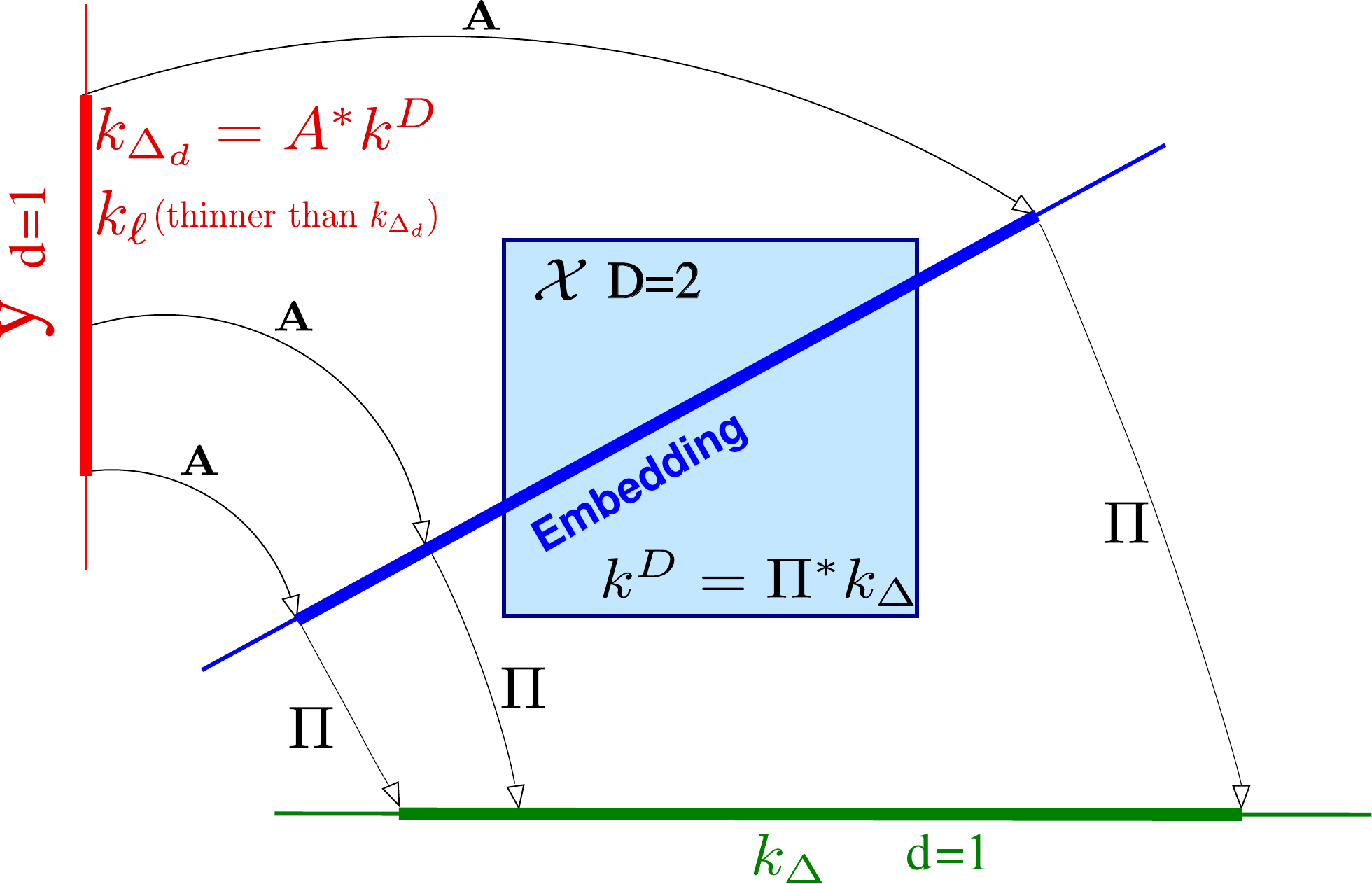}
  \caption{A illustration of the different kernels defined in this proof. 
  In this figure, 
  the red line represents the one-dimensional space $\calY$ over which we optimize.
  The blue region is the original high-dimensional space. 
  The green line represents the effective subspace $\calT$ projected to a 
  one-dimensional space through $\vPi$. The true kernel $K_{\vDelta}$ is defined on 
  the green region. The true kernel on the high-dimensional space is then naturally
  $k^D = \vPi^*k_{\vDelta}$ and the true kernel on $\calY$ is
  $K_{\vDelta_d} = \vA^*k^D$. To optimize without the knowledge of $K_{\vDelta_d}$,
  we use a kernel $k_{\ell}$ which with high probability is thinner than 
  $K_{\vDelta_d}$ in all directions thus preserving convergence properties.}
  \label{fig:kernels}
\end{figure}

\begin{mydefinition}\label{pullback} Given a map $\pi: \mathcal{S} \to \mathcal{T}$ between any two sets $\mathcal{S}$ and $\mathcal{T}$, and any map $f: \underbrace{\mathcal{T} \times \cdots \times \mathcal{T}}_\text{$n$-times} \to \mathbb R$, with $n \geq 1$, we define the \emph{pull-back} of $f$ under $\pi$ as follows:
\[ \pi^* f(s_1, \ldots, s_n) := f\left(\pi (s_1), \ldots, \pi (s_n) \right). \]
That is, one evaluates the pull-back $\pi^* f$ on points in $\mathcal{S}$ by first ``pushing them forward'' onto $\mathcal{T}$ and then using $f$ to get a number.

If the map $\pi$ is given by a matrix $\vA$, we will use the notation $\vA^* f$ for the pull-back of $f$ under the linear map induced by $\vA$. Moreover, given a matrix $\vA$ and a set $\mathcal{S}$ in its target space, we will denote by $\vA^{-1}(\mathcal{S})$ the set of all points that are mapped into $\mathcal{S}$ by $\vA$.
\end{mydefinition}

\begin{proposition}[Theorem 2 by \citeA{Bull:2011}, paraphrased for our particular setting]\label{thm:Bull} Given a squared exponential kernel $k_\ell$ on a compact subset $\mathcal Y \subset \mathbb R^d$ and a function $f \in \mathcal H_\ell(\mathcal Y)$, then applying Expected Improvement to $f$ results in simple regret that diminishes according to $\mathcal O\big(t^{-\frac{1}{d}}\big)$, with the constants worsening as the norm $\|f\|_{\mathcal H_\ell(\mathcal Y)}$ increases.
\end{proposition}

\begin{proof}[Proof of Theorem \ref{thm:regret}]
The proof of this result is structured into two parts.
In the first part of the proof, we give an analytic expression for the 
true kernel $k_{\vDelta_d}$ in the low dimensional space over which we optimize. 
In the second part of the proof, we show that this kernel is well-behaved with high probability
(specifically the maximum and minimum eigenvalues of $\vDelta_d$ are bounded above and below by constants depending on the probability of failure)
and apply Proposition~\ref{thm:Bull} to acquire the convergence rate.

Let $\vPhi \in \mathbb{R}^{D \times d}$ be a matrix, 
whose columns form an orthonormal basis for ${\cal T}$.
Let $\mathbf{\Pi}=\vPhi^T \vPhi \vPhi^T: \mathcal{X} \rightarrow \mathbb{R}^d$.
Note that $\mathbf{\Pi}$ is composed of an orthogonal projection from $\calX$ to $\calT$ and a bijective map from $\calT$ to $\mathbb{R}^d$. We will also denote the corresponding matrix by $\mathbf{\Pi}$. 

Recall from the theorem statement that $\overline{f|_{\mathcal T}}$ is assumed to be an element of the RKHS $\calH_{k_{\vDelta}}$, and that we have $f = {\vPi}^* \overline{f|_{\mathcal T}}$, i.e. $f$ is obtained from ``stretching $f|_{\mathcal T}$ open'' along the orthogonal subspace of $\mathcal T$. We can also define the kernel over $\mathbb{R}^D$ by $k^D := \mathbf{\Pi}^* k_{\vDelta}$.

Now, given the embedding $\mathbb R^d \hookrightarrow \mathbb R^D$ defined by the matrix $\vA$, the pull-back function $\vA^* f$ is an element of the RKHS $\calH_{\vA^* k^D}$: Henceforth, we will use the notation 
\[ k_{\vDelta_d} := \vA^* k^D = \vA^* \mathbf{\Pi}^* k_{\vDelta} =  (\mathbf{\Pi}\vA)^* k_{\vDelta}. \]
In more detail, for $\vy^{(1)}, \vy^{(2)} \in \mathbb{R}^d$
\[ k_{\vDelta_d}(\vy^{(1)},\vy^{(2)}) = K\left((\vy^{(1)}-\vy^{(2)})^\top \vA^\top\mathbf{\Pi}^\top\mathbf{\vDelta}^{-1}\mathbf{\Pi}\vA (\vy^{(1)}-\vy^{(2)})\right).\]
For a pictorial illustration of the different kernels defined in this proof, 
please refer to Figure~\ref{fig:kernels}.

Since $\mathbf{\Pi}$ is an orthogonal projection matrix, it has an SVD decomposition $\mathbf{\Pi} = \vU\vS\vV$ consisting of an orthogonal $d\times d$ matrix $\vU$, an orthogonal $D \times D$ matrix $\vV$ and a $d \times D$ matrix $\vS$ that has the following form:
\[ \vS = \begin{bmatrix} 1 & 0 & \cdots & 0      & 0      & \cdots & 0 \\
                       0 & 1 & \cdots & 0      & 0      & \cdots & 0 \\
            \vdots & \vdots  & \ddots & \vdots & \vdots & \ddots & \vdots \\
                       0 & 0 & \cdots & 1      & 0      & \cdots & 0 \end{bmatrix}. \]
Now, given a fixed orthogonal matrix $\vO \in \mathbb{R}^{D \times D}$ and a random Gaussian vector $\vv \sim \mathcal N(\mathbf 0,\vI_{D \times D})$, due to the rotational symmetry of the normal distribution, the vector $\vO\vv$ is also a sample from $\mathcal N(\mathbf 0,\vI_{D \times D})$. Therefore, given a random Gaussian matrix $\Gamma$, $\vO\Gamma$ is also a random Gaussian matrix with the same distribution of entries. Moreover, given $\vS$ as above, $\vS\Gamma$ is a $d \times D$ random Gaussian matrix, since multiplying any matrix by $\vS$ on the left simply extracts the first $d$ rows of the matrix.

Given this, if we fix an orthogonal decomposition $\mathbf{\vDelta}^{-1} = \vP^\top \vD^{-1} \vP$, where $\vP$ is orthogonal and $\vD$ is a diagonal matrix with the eigenvalues of $\mathbf{\vDelta}$ along the diagonal, we can conclude that
\[ \vG := \vP\mathbf{\Pi}\vA = \vP\vU\vS\vV\vA  \]
is a random Gaussian matrix, and so the matrix 
$\mathbf{\vDelta}_d = \vA^\top\mathbf{\Pi}^\top\mathbf{\vDelta}^{-1}\mathbf{\Pi}\vA$
can be decomposed into random Gaussian and diagonal matrices as follows:
\[ \mathbf{\vDelta}^{-1}_d = \vG^\top \vD^{-1} \vG. \]
Since random Gaussian matrices, as argued in the proof of Theorem~\ref{prop:1}, have full rank almost surely, $\mathbf{\vDelta}^{-1}_d$ is of full rank and is positive definite.

In the remainder of this proof, we replace $k_{\vDelta_d}$ with a kernel $k_\ell$ that is ``thinner'' than $k_{\vDelta_d}$ and so $\vA^* f$ is also an element of the RKHS of $k_\ell$. 
By showing that this is true, REMBO (which uses $k_\ell$) has enough approximation power. Moreover, the statement of Proposition~\ref{thm:Bull} applies.   

Let $s_{\min}(\vM)$ and $s_{\max}(\vM)$ denote the smallest and the largest singular values of a matrix $\vM$. With this notation in hand, 
we point out the following two facts about concentration of singular values:
\begin{itemize}
\item[I.] Since for any pair of matrices $\vA$ and $\vB$, we have $s_{\max}(\vA\vB) \leq s_{\max}(\vA) s_{\max}(\vB)$, we get
\[ \frac{1}{\lambda_{\min}(\mathbf{\vDelta}_d)} = \lambda_{\max}(\mathbf{\vDelta}_d^{-1}) \leq s_{\max}(\vG)^2 s_{\max}(\vD^{-1}) \leq \frac{s_{\max}(\vG)^2}{r^2} \]
and since $\vG$ is a random matrix with Gaussian entries, we have (cf. Equation 2.3 by \citeA{Rudelson:2010})
\[ P\left(s_{\max} (\vG) < 2\sqrt{d}+t\right) \leq 1-2e^{-t^2/2}, \]
and so with probability $1-\frac{\epsilon}{2}$, we have
\[ s_{\max}(\vG) < 2\sqrt{d}+\sqrt{2\ln\frac{4}{\epsilon}}. \]

Therefore, with probability $1-\frac{\epsilon}{2}$, we have
\begin{equation}\label{eqn:LB}
 \lambda_{\min}(\mathbf{\vDelta}_d) > \left(\frac{r}{2\sqrt{d}+\sqrt{2\ln\frac{4}{\epsilon}}}\right)^2.
\end{equation}

Henceforth, we will use the notation
\[ \ell = \ell(\epsilon) := \frac{r}{2\sqrt{d}+\sqrt{2\ln\frac{4}{\epsilon}}} \]

\item[II.] On the other hand, we have
\[ \frac{1}{\lambda_{\max}(\mathbf{\vDelta}_d)} = \lambda_{\min}(\mathbf{\vDelta}_d^{-1}) \geq s_{\min}(\vG)^2 s_{\min}(\vD^{-1}) \geq \frac{s_{\min}(\vG)^2}{R^2} \]
together with the following probabilistic bound on $s_{\min}(\vG)$ (cf. Equation 3.2 by \citeA{Rudelson:2010}):
\[ P\left( s_{\min}(\vG) > \frac{\delta}{\sqrt{d}} \right) > 1 - \delta. \]
So, with probability $1-\frac{\epsilon}{2}$, we have
\[ s_{\min}(\vG) > \frac{\epsilon}{2\sqrt{d}}, \]
and so
\begin{equation}\label{eqn:UB}
 \lambda_{\max}(\mathbf{\vDelta}_d) < \frac{4dR^2}{\epsilon^2}
\end{equation}
holds with probability $1-\frac{\epsilon}{2}$.

In what follows, we will use the notation:
\[ U = U(\epsilon) := \frac{2R\sqrt{d}}{\epsilon} \]
\end{itemize}


Now, with these estimates in hand, we have that by Lemma~\ref{lem:bull2}
and Lemma~\ref{lem:bull4}
the following bound holds with probability $1-\epsilon$:
\begin{equation}\label{eqn:RKHSNormBound}
\| \vA^*f \|_{\mathcal H_{\ell}(\vA^{-1}(\mathcal{X}))} 
\leq \| \vA^*f \|_{\mathcal H_{\ell}(\mathbb{R}^d)}
\leq \left(\frac{U(\epsilon)}{\ell(\epsilon)}\right)^{\frac{d}{2}} 
     \|\vA^* f\|_{\mathcal H_{\mathbf{\vDelta}_d}(\mathbb{R}^d)}
\end{equation}

Since the transformation $\mathbf{\Pi}\vA$ is invertible, we have that the map $(\mathbf{\Pi}\vA)^*: \mathcal H_{\mathbf{\vDelta}}(\mathbb R^d) \to \mathcal H_{\mathbf{\vDelta}_d}(\mathbb R^d)$ (recall that $k_{\mathbf{\vDelta}_d} = k_{(\mathbf{\Pi}\vA)^*\mathbf{\vDelta}}$) that sends $g \in H_{\mathbf{\vDelta}}$ to $(\mathbf{\Pi}\vA)^* g$ is an isomorphism of Hilbert spaces and so 
\begin{equation}\label{eqn:KernelIso}
\left\|\overline{f|_{\mathcal T}}\right\|_{\mathcal H_\mathbf{\vDelta}(\mathbb R^d)} = \|\vA^* f\|_{\mathcal H_{\mathbf{\vDelta}_d}(\mathbb R^d)}
\end{equation}
since we have $\vA^* f = \vA^* \left(\mathbf{\Pi}^* \overline{f|_{\cal T}}\right) = (\mathbf{\Pi}\vA)^* \overline{f|_{\cal T}}$.

By combining~\ref{eqn:RKHSNormBound} and~\ref{eqn:KernelIso}, we have that
$$
\| \vA^*f \|_{\mathcal H_{\ell}(\vA^{-1}(\mathcal{X}))} 
\leq \left(\frac{U(\epsilon)}{\ell(\epsilon)}\right)^{\frac{d}{2}} 
     \left\| \overline{f|_{\mathcal T}} \right\|_{\mathcal H_{\mathbf{\vDelta}}(\mathbb{R}^d)}.
$$

Now that we know that the $\mathcal H_{\ell}(\mathbb{R}^d)$ norm of $\vA^* f$ is finite, we can apply the Expected Improvement algorithm to it on the set $\vA^{-1}(\mathcal X)$ with kernel $k_\ell$, instead of the unknown kernel $k_{\vDelta_d}$, and then Proposition \ref{thm:Bull} tells us that the simple regret would be in $\mathcal O\big(t^{-\frac{1}{d}}\big)$.
\end{proof}


\end{document}